\documentclass[sigconf]{acmart}

\setcopyright{none}
\acmConference{MILETS}{2026}{}
\acmISBN{}
\acmDOI{}
\renewcommand\footnotetextcopyrightpermission[1]{}
\usepackage{booktabs}
\usepackage{multirow}
\usepackage{graphicx}
\usepackage{dsfont}
\usepackage{enumitem}
\usepackage{amsmath}

\usepackage{amssymb}
\usepackage{amsthm}
\usepackage{tikz}
\usepackage{adjustbox}
\usepackage{float}  
\usetikzlibrary{shapes.geometric, arrows.meta, positioning, calc}

\theoremstyle{definition}
\newtheorem{definition}{Definition}
\theoremstyle{plain}
\newtheorem{fact}{Fact}
\newtheorem{proposition}{Proposition}

\tikzstyle{stepbox} = [rectangle, rounded corners=4pt, minimum width=3.4cm, minimum height=0.7cm,
                       text centered, draw=black!60, fill=blue!10, font=\small]
\tikzstyle{decision} = [diamond, aspect=2.4, minimum width=2.0cm, minimum height=0.6cm,
                        text centered, draw=black!60, fill=yellow!15, font=\small, inner sep=1pt]
\tikzstyle{tryterm} = [rectangle, rounded corners=4pt, minimum width=3.4cm, minimum height=0.9cm,
                       text centered, text width=3.2cm, draw=black!60, fill=green!20, font=\small, align=center]
\tikzstyle{skipterm} = [rectangle, rounded corners=4pt, minimum width=3.4cm, minimum height=0.9cm,
                        text centered, text width=3.2cm, draw=black!60, fill=red!15, font=\small, align=center]
\tikzstyle{inconterm} = [rectangle, rounded corners=4pt, minimum width=3.4cm, minimum height=0.9cm,
                         text centered, text width=3.2cm, draw=black!60, fill=gray!20, font=\small, align=center]
\tikzstyle{flowarrow} = [thick, -{Latex[length=2mm]}]

\begin{document}

\title{When Does Context Routing Help? A Systematic Study of Multi-Modal Fusion in Time Series Forecasting}

\author{Ruizhe Zhou}
\affiliation{%
  \institution{Amazon.com}
  \city{Seattle}
  \country{USA}
}
\email{rexzhou@amazon.com}

\author{Gaoyuan Du}
\affiliation{%
  \institution{University of Tennessee, Knoxville}
  \city{Knoxville}
  \country{USA}
}
\email{gdu3@vols.utk.edu}

\author{Xiaoyang Liu}
\affiliation{%
  \institution{Amazon.com}
  \city{Seattle}
  \country{USA}
}
\email{lxaoya@amazon.com}

\author{Haoqi Yao}
\affiliation{%
  \institution{WorkMagic}
  \city{New York}
  \country{USA}
}
\email{byahak@gmail.com}

\author{Deepayan Chakrabarti}
\affiliation{%
  \institution{University of Texas at Austin}
  \city{Austin}
  \country{USA}
}
\email{deepay@utexas.edu}

\author{Jiating Lin}
\affiliation{%
  \institution{Amazon.com}
  \city{Seattle}
  \country{USA}
}
\email{ljtlinus@amazon.com}

\author{Yixuan Shen}
\affiliation{%
  \institution{Amazon.com}
  \city{Chicago}
  \country{USA}
}
\email{jolina.shen@gmail.com}

\begin{abstract}
Multi-modal time series forecasting methods integrate auxiliary context (text, tabular features, financial signals) into temporal predictions through increasingly sophisticated fusion mechanisms. A growing body of work reports substantial gains, yet it is often unclear whether they reflect genuine use of the context or incidental architectural effects (extra parameters, regularization, residual paths). We ask a narrower, checkable question: \emph{when can auxiliary context help a forecaster at all?}

We identify two dataset-level conditions that must both hold: (1) the target is not dominated by a last-value shortcut (low autocorrelation $\rho_h$), and (2) the context carries information about the target beyond history (non-zero conditional mutual information $\delta = I(C; X_{t+h}\mid X_t)$; when $\delta=0$ no predictor can benefit---a distribution-free result). Through controlled experiments on MoME (a 14.3B-parameter mixture-of-experts model, 6 datasets, 10 seeds) and four additional fusion mechanisms implemented within a single-backbone testbed (5 datasets), we find that when both conditions hold, text-conditioned expert modulation contributes a sizeable MSE reduction (HealthUS +51\%, Environment +41.5\%, SocialGood +32\%, HealthAFR +29\%); when either fails, the contribution collapses to the capacity floor of the modulation pathway and carries no context-attributable signal.

We establish causality through two interventions: adding a shortcut to MoME suppresses routing contribution by 77--93\% across 3 datasets; progressively corrupting context quality drives the context-specific benefit from +44\% to negative. We validate the autocorrelation component of our diagnostic on 27 Monash Archive datasets (Spearman $r = 0.888$, $p < 0.0001$). We provide a calibrated pre-training diagnostic (TRY\_FUSION / SKIP\_FUSION / INCONCLUSIVE) that, on the datasets we test, yields no false positives in well-powered settings. We are explicit about the asymmetry of our evidence: the negative (SKIP) arm is broadly reliable, while the large positive magnitudes come from a single model family (MoME) and are corroborated only in \emph{direction} by the testbed.
\end{abstract}


\maketitle

\section{Introduction}

\subsection{Background and Motivation}

Time series forecasting is a fundamental task in domains ranging from finance and energy to healthcare and climate science. Recent years have seen substantial progress through Transformer-based architectures~\cite{zhou2021informer,wu2021autoformer,nie2023patchtst,liu2024itransformer}, simple linear baselines~\cite{zeng2023dlinear}, and large-scale time series foundation models~\cite{ansari2024chronos,das2024timesfm,woo2024moirai,goswami2024moment}. Recently, \emph{multi-modal} forecasting methods have emerged that augment temporal data with auxiliary context (news text, weather reports, economic indicators, event calendars) with the goal of improving prediction accuracy~\cite{survey2025mmts}.

These methods employ various \emph{fusion mechanisms} to integrate context into the forecasting pipeline:
\begin{itemize}[leftmargin=*]
    \item \textbf{Output-level fusion}~\cite{liu2024timemmd}: separate models process time series and context independently; their predictions are combined (e.g., added) at the output layer.
    \item \textbf{Cross-attention alignment}~\cite{liu2025timecma}: time series and context representations interact through cross-attention in a shared latent space.
    \item \textbf{Text-as-variable}~\cite{li2025tats}: text embeddings are concatenated with temporal patches as additional input variables.
    \item \textbf{Gating}~\cite{survey2025mmts,perez2018film}: context modulates temporal features through learned gates (e.g., FiLM-style affine transformations).
    \item \textbf{Mixture-of-experts (MoE) routing}~\cite{mome2025,shazeer2017outrageously}: context conditions the selection or modulation of specialized expert sub-networks.
\end{itemize}

A common narrative in the literature is that more sophisticated fusion leads to better performance. However, reported gains vary dramatically across datasets and settings, and it remains unclear whether improvements stem from the fusion mechanism exploiting context information, or from incidental architectural benefits (e.g., additional parameters, regularization effects, or residual connections).

\subsection{The Problem}

We address a simple but important question: \textbf{under what conditions does auxiliary context provide genuine value to a time series forecaster?}

By ``genuine value,'' we mean improvement attributable to the \emph{context information} itself, not architectural side effects. This distinction matters for practitioners: if fusion helps only because of better regularization (not context), then simpler architectural changes would achieve the same benefit at lower cost.

\subsection{Key Findings}

Through controlled experiments, we identify two conditions that must both hold:

\begin{enumerate}[leftmargin=*]
    \item \textbf{No competing shortcut.} Many forecasting architectures include a ``last-value shortcut'': a residual connection that directly copies the most recent observation as a baseline prediction. The phenomenon mirrors the broader notion of \emph{shortcut learning}~\cite{geirhos2020shortcut}, in which models exploit easy predictive signals at the expense of more nuanced ones. When present, this shortcut captures most of the predictable variance on autocorrelated data, leaving little room for fusion to contribute. We show that adding a shortcut to MoME suppresses its routing contribution by 77--93\% across 3 datasets (HealthUS, SocialGood, HealthAFR).
    
    \item \textbf{Statistically informative context.} Context must contain information about the target beyond what the recent history already provides. We operationalize this via a mutual information (MI) permutation test on the residual $R = X_{t+h} - \hat\rho_h X_t$: if context MI is non-significant ($p > 0.05$) in a well-powered setting, no fusion mechanism helps, regardless of model capacity. By Fact~\ref{fact:null}, when the true conditional MI is zero this holds for \emph{any} predictor and distribution.
\end{enumerate}

When both conditions hold, text modulation contributes a sizeable MSE reduction on the 14.3B-parameter MoME model (HealthUS +51\%, Environment +41.5\%, SocialGood +32\%, HealthAFR +29\%; all 10-seed). When either fails, the contribution collapses to the modulation pathway's capacity floor. We stress at the outset that these large \emph{magnitudes} come from a single model family; the testbed corroborates the \emph{direction} but not the size.

\subsection{Contributions}

\begin{enumerate}[leftmargin=*]
    \item \textbf{Causal decomposition.} We isolate the fusion mechanism's contribution from architectural confounds through three controlled interventions: toggling text modulation (same model, different datasets), adding/removing shortcuts (replicated on 3 datasets with 77--93\% suppression), and corrupting context quality (same dataset, degraded input). Each intervention demonstrates a causal relationship, not merely a correlation.
    
    \item \textbf{Multi-scale validation.} We validate the diagnostic's predictions at two scales: (i) within a controlled testbed, 4 fusion mechanisms confirm that MI-non-significant datasets never show positive context benefit; (ii) across 27 Monash Archive datasets, autocorrelation at the prediction horizon predicts shortcut dominance with Spearman $r = 0.888$ ($p < 0.0001$).
    
    \item \textbf{Calibrated pre-training diagnostic.} We provide a 3-output diagnostic (TRY\_FUSION / SKIP\_FUSION / INCONCLUSIVE) based on autocorrelation and MI significance, with formal power characterization. On the datasets we test, it yields no false positives in well-powered settings and correctly outputs INCONCLUSIVE (rather than a misleading SKIP) when statistical power is insufficient.
\end{enumerate}

\paragraph{Scope.} Our diagnostic is modality-agnostic: it operates on the context as a feature vector $C$ and makes no assumption about its source. The datasets we evaluate provide context as embeddings of textual side information (Time-MMD) or financial news (FinMultiTime); we do not test image or audio modalities. We use ``fusion'' throughout in this embedding-level sense.

\section{Theoretical Foundation}
\label{sec:theory}

We establish the information-theoretic basis for our diagnostic, drawing on standard results in information theory~\cite{cover2006elements}. The basis of our approach is that two dataset properties (autocorrelation and context informativeness) jointly determine the \emph{maximum possible} benefit from fusion.

\subsection{Setup and Definitions}

Consider a stationary time series $\{X_t\}$ with marginal variance $\sigma^2$. We aim to predict $X_{t+h}$ (the value $h$ steps ahead) given:
\begin{itemize}
    \item $X_t$: the most recent observation (``history'')
    \item $C$: an auxiliary context variable (e.g., text embedding, tabular features)
\end{itemize}

\begin{definition}[$h$-step autocorrelation]
$\rho_h = \text{Corr}(X_t, X_{t+h})$ measures how predictable the future is from the past alone. When $\rho_h \approx 1$ (e.g., daily stock prices), simply repeating the last value is already a strong prediction.
\end{definition}

\begin{definition}[Conditional mutual information]
$\delta = I(C; X_{t+h} \mid X_t)$ (in bits) measures how much \emph{additional} information context $C$ provides about the target $X_{t+h}$, beyond what history $X_t$ already reveals. When $\delta = 0$, context is conditionally independent of the target given history, and it cannot help any predictor.
\end{definition}

\begin{definition}[Last-value shortcut]
A model component that directly uses $X_t$ (or a linear function of recent values) as a baseline prediction: $\hat{X}_{t+h}^{\text{shortcut}} = f_{\text{linear}}(X_t)$. This achieves MSE $= \sigma^2(1-\rho_h^2)$ for the optimal linear predictor.
\end{definition}

\subsection{When Can Context Help? (Distribution-Free Results)}

\begin{fact}[Null condition --- distribution-free]
\label{fact:null}
If $\delta = I(C; X_{t+h} \mid X_t) = 0$, then $C$ is conditionally independent of $X_{t+h}$ given $X_t$. In this case, \emph{no predictor}, regardless of capacity, architecture, or training procedure, can benefit from $C$. Formally: $\text{MMSE}(X_{t+h} \mid X_t, C) = \text{MMSE}(X_{t+h} \mid X_t)$.
\end{fact}

\noindent This is the theoretical foundation of our MI permutation test: if we cannot reject $\delta = 0$, fusion is provably futile.

\begin{fact}[Shortcut ceiling --- distribution-free]
\label{fact:ceiling}
The best linear predictor from $X_t$ alone achieves $\text{MSE}_{\text{shortcut}} = \sigma^2(1-\rho_h^2)$. Under Gaussianity, this equals the MMSE (the linear predictor is optimal). For non-Gaussian data, the MMSE may be strictly lower (nonlinear predictors can do better), but the linear shortcut MSE still upper-bounds the gap that context can fill: any improvement from context is at most $\sigma^2(1-\rho_h^2)$. When $\rho_h \to 1$, even this upper bound vanishes, leaving negligible room for context regardless of distribution.
\end{fact}

\subsection{Quantifying the Opportunity (Gaussian Benchmark)}

\begin{proposition}[Routing Benefit Upper Bound --- RBU]
\label{prop:rbu}
Under joint Gaussianity of $(X_t, X_{t+h}, C)$, the maximum MSE reduction from adding context is exactly:
\begin{equation}
    \text{RBU} = \sigma^2(1-\rho_h^2)(1-2^{-2\delta})
\end{equation}
This decomposes into two factors: $(1-\rho_h^2)$ is the residual variance \emph{available} for context to explain (the ``room'' left by the shortcut), and $(1-2^{-2\delta})$ is the fraction of that room that context \emph{can} explain. (Proof in Appendix~\ref{app:rbu_proof}.)
\end{proposition}

\noindent The RBU formula is exact only under joint Gaussianity (proof in Appendix~\ref{app:rbu_proof}). Off Gaussianity, two separate caveats apply, which we keep distinct. (i) \emph{Formula:} the population RBU expression is in general \emph{neither an upper nor a lower bound} on the true achievable MMSE reduction, because the Gaussian entropy--variance identity it relies on no longer holds. (ii) \emph{Estimator:} our plug-in uses a finite-sample Kraskov estimate $\hat\delta$, which under-detects nonlinear dependence (Finding~\ref{sec:nonlinear}) and is therefore biased \emph{low}, so $\widehat{\text{RBU}}$ tends to underestimate the population RBU. We therefore use RBU only as a Gaussian-case sanity check for the relative magnitude of routing benefit, not as a bound. The two distribution-free components of our diagnostic are Fact~\ref{fact:null} (no benefit when $\delta=0$) and Fact~\ref{fact:ceiling} (shortcut ceiling as $\rho_h\to 1$).

\subsection{MI Estimation and Permutation Test}

We estimate $\delta$ using a three-step procedure:
\begin{enumerate}
    \item Compute the residual $R = X_{t+h} - \hat{\rho}_h X_t$ (removing the linear shortcut component).
    \item Project the high-dimensional context $C \in \mathbb{R}^d$ to its top-20 PCA components (reducing dimensionality while preserving most variance).
    \item Estimate $\hat{\delta} = \hat{I}(C_{\text{PCA}}; R)$ using the Kraskov k-nearest-neighbor estimator~\cite{kraskov2004estimating} with $k=3$, converted to bits.
\end{enumerate}

The k-NN estimator is \emph{nonparametric}: it detects nonlinear dependencies without assuming any functional form between context and target. This is critical: we show empirically (Section~\ref{sec:nonlinear}) that context-target relationships are fundamentally nonlinear, making linear diagnostics (e.g., cross-validated linear regression) completely uninformative. Alternative MI estimators (e.g., MINE~\cite{belghazi2018mine} and variational bounds~\cite{poole2019variational}) trade tractability for flexibility; we choose Kraskov k-NN for its established convergence properties and simple permutation-test integration.

\paragraph{Significance testing.} To determine whether $\hat{\delta}$ is statistically distinguishable from zero, we perform a permutation test: shuffle context rows (breaking any dependence with the target), re-estimate MI, and repeat 200 times. The $p$-value is the fraction of null MI values $\geq$ observed MI. By Fact~\ref{fact:null}, if the true $\delta = 0$, the conclusion ``fusion is futile'' holds for \emph{any} distribution.

\paragraph{Power limitation.} The k-NN MI estimator requires $n \gtrsim O(d_{\text{eff}} / \delta^2)$ samples with unique context to reliably detect effect size $\delta$~\cite{berrett2019nonparametric}. On small datasets ($n < 500$) or with low context diversity (few unique embeddings reused cyclically), the test may lack power, producing false negatives. It does \emph{not} produce false positives: in the well-powered settings we test, non-significant MI reliably predicts negligible routing benefit.

\paragraph{Conditioning on a single lag.} Our residual $R = X_{t+h} - \hat\rho_h X_t$ removes only the \emph{first-order} linear dependence on the most recent observation. If the series carries longer-memory structure (seasonality, higher-order autoregression) that a richer history $X_{t-k:t}$ would capture, that predictable variance remains in $R$, and context that merely correlates with it can inflate $\hat\delta$. Our test therefore assesses informativeness \emph{beyond a last-value baseline}, not beyond an optimal history-based predictor; on strongly seasonal data the TRY arm may overstate context value. The distribution-free null (Fact~\ref{fact:null}) is unaffected---when context is conditionally independent of the target it remains so under any history---but conditioning the residual on a multi-lag history is a natural and recommended extension for series with strong seasonality.

\section{Experimental Setup}

\subsection{Datasets}

We evaluate on 8 datasets spanning diverse domains, autocorrelation profiles, and sample sizes:

\begin{itemize}[leftmargin=*]
    \item \textbf{Time-MMD} (6 sub-datasets): HealthUS (quarterly US health metrics, $n$=491, $\rho$=0.77), HealthAFR (African health surveillance, $n$=500), Energy (weekly energy prices, $n$=1059, $\rho$=0.99), Environment (environmental monitoring, $n$=1597, $\rho$=0.38), SocialGood (social indicators, $n$=363, $\rho$=0.62), Web (web traffic, $n$=1068, $\rho$=0.12). All include text context (384-dimensional sentence embeddings).
    
    \item \textbf{FinMultiTime}~\cite{chen2025finmultitime}: 50 S\&P500 stocks with daily prices ($\rho > 0.99$) and news embeddings. Represents the extreme high autocorrelation regime where shortcuts dominate.
\end{itemize}

\subsection{Models}

\paragraph{MoME~\cite{mome2025}.} A state-of-the-art multi-modal forecasting model built on Qwen1.5-MoE-A2.7B (14.3B total parameters). MoME uses a mixture-of-experts architecture where each expert processes time series patches. The mechanism we ablate is \textbf{Expert-level Language Modulation (EiLM)}: after each expert produces its output, a FiLM-style layer~\cite{perez2018film} applies text-conditioned affine transformation ($\gamma \cdot x + \beta$, where $\gamma, \beta$ are derived from text tokens). Toggling the \texttt{--modulation} flag enables/disables EiLM (and its associated instructor QueryPool) while keeping expert structure and routing weights unchanged ($\sim$74K modulation parameters out of 14.3B, $<$0.001\%). Note that MoME has \emph{no built-in last-value shortcut}: the MoE pathway is the primary prediction mechanism.

\paragraph{Routing testbed.} A controlled ablation tool (not a proposed method) that serves as a \emph{mechanistic probe}: while MoME demonstrates the magnitude of routing benefits in realistic settings, the testbed isolates specific mechanisms (shortcut presence, routing on/off) under controlled conditions. We implement it using a PatchTransformer backbone (2 layers, $d$=64) with: (a) toggleable sparse routing via entmax~\cite{peters2019sparse} + blend gate, (b) an optional last-value shortcut (linear layer initialized as repeat-last-value), and (c) context dropout. ``Routing'' = routing enabled; ``Uniform'' = routing weights fixed to $1/K$.

\paragraph{Multi-modal baselines.} We implement 4 fusion mechanisms within the same backbone for controlled comparison: Cross-Attention Alignment~\cite{liu2025timecma}, Gating~\cite{survey2025mmts}, Text-as-Variable~\cite{li2025tats}, and Output Fusion~\cite{liu2024timemmd}. Each is tested with real context vs.\ zeroed context (same architecture, only input differs).

\paragraph{Foundation model.} Chronos-T5-Small~\cite{ansari2024chronos} provides a zero-shot reference point (no context, no training on our data).

\subsection{Evaluation Protocol}

\paragraph{Metrics.} MSE and MAE for MoME (point forecasting); weighted quantile loss (wQL) at $\tau \in \{0.1, 0.5, 0.9\}$ for the testbed (probabilistic forecasting). All results report 3-seed mean $\pm$ std unless noted.

\paragraph{Routing contribution.} Defined as
\[
\frac{\text{metric}_{\text{w/o mod}} - \text{metric}_{\text{w/ mod}}}{\text{metric}_{\text{w/o mod}}} \times 100\%.
\]
Positive values indicate that text modulation improves performance.

\section{Results}
\label{sec:results}

We present six findings that collectively characterize when and why multi-modal fusion helps. Findings 1--5 isolate specific causal factors through controlled experiments; Finding 6 validates the $\rho$-based component of our diagnostic at scale on 27 Monash Archive datasets.

\subsection{Finding 1: Context Informativeness Determines Routing Value}
\label{sec:finding1}

Our first experiment holds the model constant and varies only the dataset. If routing value is determined by dataset properties (rather than architectural details), the same model should show dramatically different contributions across datasets.

\begin{table}[t]
\centering
\caption{MoME routing contribution across 6 datasets, each on a 10-seed protocol with 95\% CIs (footnotes). All experiments use identical model configuration (Qwen1.5-MoE-A2.7B, 14.3B params). The only variable is the dataset. The ``Naive MSE'' column reports the post-hoc MSE of a repeat-last-value predictor (no model training); it serves as a context-free reference for whether MoME genuinely beats a trivial baseline. This naive baseline is conceptually distinct from the architectural shortcut intervention used in Finding 2, which adds a residual connection during MoME training. Naive MSE values are reported only for the two datasets where the comparison is most informative (HealthUS, where MoME wins; FinMultiTime, where the naive baseline beats MoME).}
\label{tab:mome}
\small
\resizebox{\columnwidth}{!}{
\begin{tabular}{lccccc}
\toprule
Dataset & With Mod. & Without Mod. & Contrib. & Naive MSE & $\rho$ \\
\midrule
HealthUS & .399$\pm$.129 & .817$\pm$.049 & \textbf{+50.9\%} & .409 & .77 \\
Environment$^*$ & 12.63$\pm$0.39 & 21.59$\pm$0.15 & \textbf{+41.5\%} & --- & .38 \\
SocialGood & .398$\pm$.032 & .597$\pm$.086 & \textbf{+32.0\%} & --- & .62 \\
HealthAFR & .639$\pm$.082 & .898$\pm$.055 & \textbf{+28.7\%} & --- & .11 \\
\midrule
FinMultiTime & 9.1$\pm$0.2e-5 & 9.8$\pm$0.2e-5 & +6.7\% & \textbf{8.8e-5} & .999 \\
Energy & 9.6$\pm$1.2e-3 & 10.2$\pm$0.9e-3 & +4.3\%$^\sharp$ & --- & .99 \\
\bottomrule
\end{tabular}}

\smallskip
\noindent\scriptsize $^*$Environment uses MAE (different task output format); all others use MSE. Contribution = (Without Mod.$-$With Mod.)/Without Mod.$\times$100\% (per-seed mean over 10 seeds). 95\% CIs: HealthUS [40.0,61.8]; Environment [40.3,42.7]; SocialGood [25.0,39.1]; HealthAFR [22.6,34.7]; FinMultiTime [5.1,8.3]. $^\sharp$Energy's routing contribution is below the 5\% threshold and \emph{not statistically distinguishable from zero} (10-seed mean +4.3\%, 95\% CI [$-$8.0,16.5]; per-seed values span $-$27\% to +27\%). With $\rho{=}0.99$ (shortcut dominates) and non-significant MI, the diagnostic correctly outputs SKIP.
\end{table}

Table~\ref{tab:mome} confirms this prediction. The same 14.3B-parameter model, trained with the same code and hyperparameters, shows routing contributions ranging from the capacity floor (high-$\rho$ datasets) to +51\% (HealthUS). The four datasets with large positive contributions (HealthUS +51\%, Environment +41.5\%, SocialGood +32\%, HealthAFR +29\%) share two properties: moderate-to-low autocorrelation ($\rho < 0.8$) and text context that describes conditions relevant to the forecast target. The two high-autocorrelation datasets ($\rho \geq 0.99$) show only capacity-floor-level contributions: their MI is non-significant, and the routing contribution sits at or below the level the EiLM pathway yields from capacity alone even with constant text. FinMultiTime is +6.7\% (and a trivial shortcut already beats MoME outright), while Energy is \emph{not statistically distinguishable from zero}.

\paragraph{Why does FinMultiTime show negligible benefit?} On FinMultiTime, a simple repeat-last-value predictor (the ``Naive MSE'' baseline in Table~\ref{tab:mome}) achieves MSE $\sim$$8.8 \times 10^{-5}$, \emph{better} than MoME with modulation ($\sim$$9.1 \times 10^{-5}$); the routing contribution is +6.7\%. The temporal signal is so strong ($\rho = 0.999$) that even a 14.3B-parameter model cannot beat the trivial baseline. There is simply nothing left for context to improve.

\paragraph{Causal validation: context degradation.} To confirm that context quality \emph{causally} determines routing value (rather than merely correlating with it), we progressively corrupt HealthUS's text context by randomly replacing text entries with uninformative placeholders at rates of 0\%, 25\%, 50\%, and 100\%.

\paragraph{Masking protocol.} HealthUS has 491 samples with 491 distinct texts. For each sample we replace its text with a fixed uninformative placeholder with probability $p \in \{0, 0.25, 0.5, 1.0\}$. At $p{=}1.0$ all samples share identical text, so EiLM receives a constant embedding and can only learn a fixed affine transform. The same masked dataset is used at train and test.

\begin{figure}[t]
    \centering
    \includegraphics[width=0.85\columnwidth]{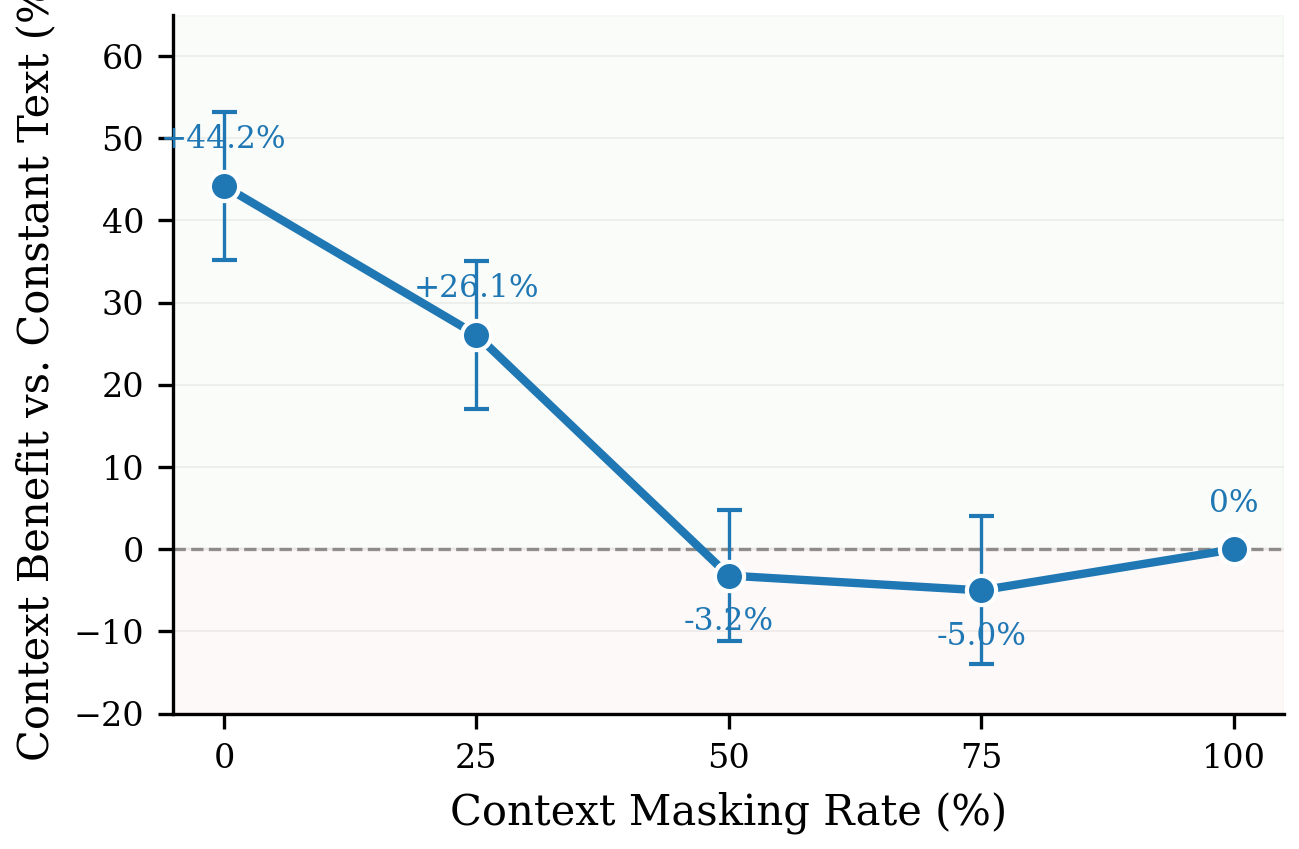}
    \caption{Context benefit relative to constant-text baseline (MoME on HealthUS, 10-seed; same architecture, same parameters, only text content differs). At 0\% mask (all real text), context provides a +44\% MSE reduction over constant text. As text is corrupted, benefit decreases (+26\% at 25\% mask) then turns negative at 50--75\% mask: partially corrupted text is \emph{worse} than fully constant text because the model attempts to use text but is misled by fake entries. At 100\% mask (baseline), all text is identical and the model learns to ignore it.}
    \label{fig:degradation}
\end{figure}

\paragraph{Result.} Figure~\ref{fig:degradation} shows context benefit using a capacity-controlled metric: we compare MSE at each mask rate (modulation ON) against the 100\% mask baseline (modulation ON, constant text), where the architecture and parameter count are identical and only the text content differs. With clean context (0\% mask), real text reduces MSE by 44\% relative to constant text. At 25\% mask, benefit drops to 26\%. At 50--75\% mask, benefit turns \emph{negative} ($-$3 to $-$5\%): partially corrupted text is worse than fully constant text because the model attempts to exploit text (some entries appear real) but is misled by the placeholder entries. At 100\% mask, all text is identical and the model learns a stable fixed transform (baseline = 0\% by definition).

\paragraph{Interpretation.} The negative benefit at 50--75\% mask reveals an important practical insight: \emph{low-quality context is worse than no context}. When text is partially corrupted, the model cannot distinguish informative from uninformative entries and is actively harmed. This supports our diagnostic's conservative design: when context informativeness is uncertain (INCONCLUSIVE), practitioners should either use fully clean context or disable the context pathway entirely, not feed noisy context hoping for partial benefit.

\subsection{Finding 2: Shortcuts Causally Suppress Routing Contribution}
\label{sec:finding2}

Our second experiment holds the dataset constant and varies the architecture. Specifically, we test whether adding a last-value shortcut to MoME suppresses routing contribution, as predicted from our theoretical framework (Fact~\ref{fact:ceiling}).

\paragraph{Experiment design.} We modify MoME's training to include a residual shortcut: the model's output becomes $\hat{y} = f_{\text{MoME}}(X, C) + X_t$ (adding the last observed value). This gives the model a ``free'' baseline prediction, meaning the learned component $f_{\text{MoME}}$ only needs to predict the \emph{residual} beyond the shortcut. We train 4 conditions on each dataset: \{with/without modulation\} $\times$ \{with/without shortcut\}, each with 3 seeds.

\begin{figure}[t]
    \centering
    \includegraphics[width=0.85\columnwidth]{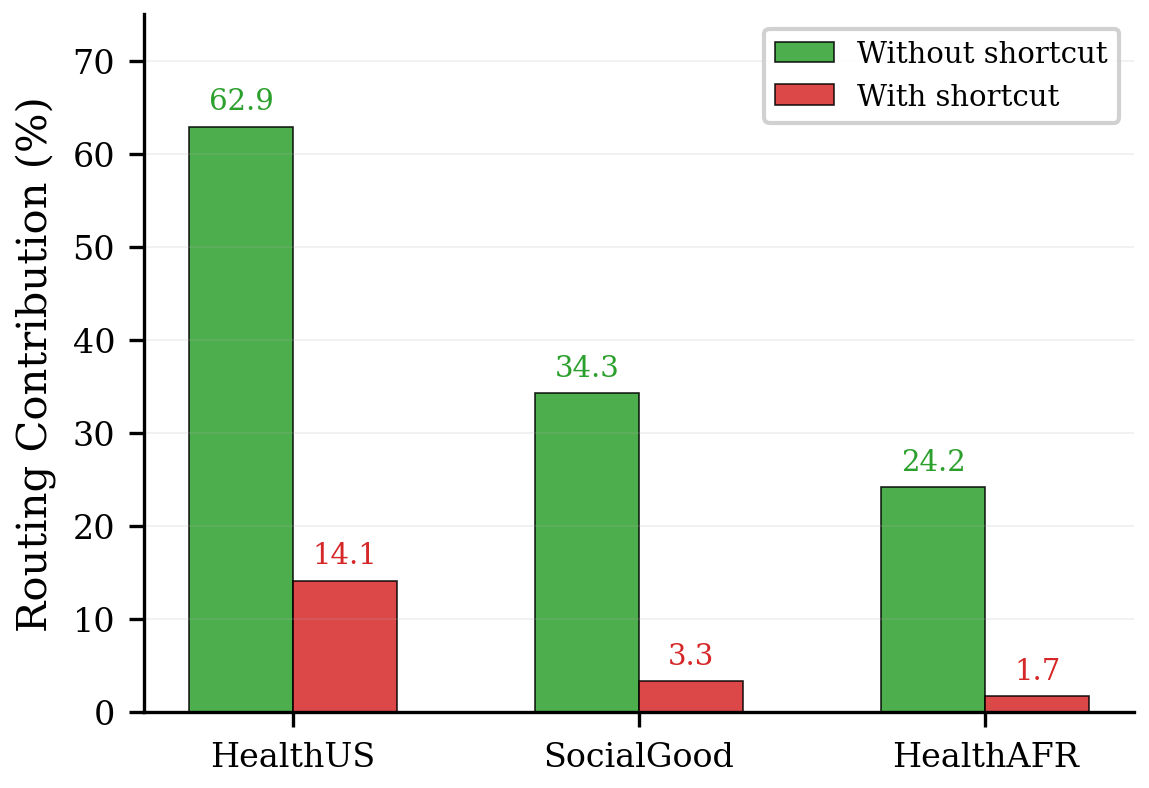}
    \caption{Shortcut suppression across 3 datasets (MoME, 3 seeds each). Adding a last-value shortcut during training consistently suppresses routing contribution by 77--93\%. Suppression is near-total on SocialGood and HealthAFR (residual contribution $<$4\%); on HealthUS, the moderate autocorrelation ($\rho = 0.77$) leaves room for a residual 14\% routing contribution.}
    \label{fig:shortcut}
\end{figure}

\paragraph{Result.} Figure~\ref{fig:shortcut} shows the result across 3 datasets: routing contribution is consistently suppressed by 77--93\% when a shortcut is added. On HealthUS, contribution drops from \textbf{+62.9\%} to \textbf{+14.1\%}; on SocialGood, from \textbf{+34.3\%} to \textbf{+3.3\%}; on HealthAFR, from \textbf{+24.2\%} to \textbf{+1.7\%}. The shortcut absorbs predictable variance that would otherwise be captured by text modulation, leaving less room for routing to contribute. The suppression is not complete on HealthUS ($\rho = 0.77$) because the shortcut is imperfect at moderate autocorrelation, but it is near-total on SocialGood and HealthAFR (residual contribution $<$4\%).

\paragraph{Why does the shortcut increase absolute MSE? (and a caveat).} On HealthUS the raw shortcut raises absolute MSE (0.33/0.90 $\to$ 1.09/1.27): the data is z-normalized, so on trending windows the optimal coefficient is $\beta_h>1$ and a raw $+X_t$ term is biased. \emph{We flag this as a limitation}: the suppression magnitudes in Fig.~\ref{fig:shortcut} should be read as an \emph{upper bound}. Re-running with the unbiased optimal shortcut $\beta_h X_t$ keeps MSE in range and still suppresses routing partially (+62.9\%\,$\to$\,+35.6\%, consistent with the residual room at $\rho{=}0.77$), so the causal claim holds under both; the relevant comparison is always \emph{within} each condition (modulation on vs.\ off).

\paragraph{Testbed confirmation.} Our controlled testbed confirms the same pattern across 5 datasets (Table~\ref{tab:shortcut_ablation}). With shortcut present, routing benefit is statistically indistinguishable from zero on \emph{all} datasets (including a 10-seed evaluation on Transport: $-2.2\% \pm 3.9\%$, $p = 0.13$). Removing the shortcut unmasks the true pattern: positive benefit on MI-significant datasets, near-zero on MI-non-significant datasets.

\begin{table}[t]
\centering
\caption{Testbed routing benefit with vs.\ without shortcut (3-seed wQL $\downarrow$). With shortcut present, routing contributes $\leq$0\% on all datasets regardless of MI significance. Removing the shortcut reveals the underlying pattern.}
\label{tab:shortcut_ablation}
\small
\begin{tabular}{lccc}
\toprule
Dataset & With Shortcut & Without Shortcut & MI sig.? \\
\midrule
Health & $-$0.6\% ($p$=.34) & \textbf{+4.8\%} & Yes ($p<$.001) \\
Transport & $-$2.2\% ($p$=.13) & \textbf{+2.2\%} & Yes ($p$=.015) \\
Energy & $-$1.6\% & +1.8\% & No ($p$=.580) \\
Climate & $-$1.4\% & $-$0.5\% & No ($p$=.705) \\
Web & $-$3.5\% & $-$0.4\% & No ($p$=1.0) \\
\bottomrule
\end{tabular}
\end{table}

\paragraph{Negative control: FinMultiTime.} A natural concern is whether our shortcut intervention always works, or only on the three datasets we tested. FinMultiTime ($\rho = 0.999$) provides a built-in negative control. The naive repeat-last-value baseline already achieves MSE $8.8 \times 10^{-5}$ versus MoME's $9.1 \times 10^{-5}$ (Table~\ref{tab:mome})---i.e., the simplest possible context-free predictor already outperforms MoME. Adding an architectural shortcut on top would not change routing contribution further: it is already $\approx 0$ because the temporal signal completely dominates and there is no residual variance for routing to exploit. This serves as a saturation case: when the naive baseline already wins, the suppression mechanism we identified has no room to operate.

\paragraph{Practical implication.} This finding has direct architectural consequences. Many modern forecasting models include residual connections or normalization layers (e.g., RevIN~\cite{kim2022reversible}) that function as implicit shortcuts. If practitioners want to measure whether context \emph{genuinely} helps, they must control for this confound, either by removing the shortcut or by comparing against a shortcut-only baseline.

\subsection{Finding 3: Context Contribution Across Fusion Mechanisms}
\label{sec:finding3}

Our third experiment asks: does the ``MI-significant $\rightarrow$ fusion helps'' pattern hold across different fusion mechanisms, or is it specific to MoME's expert modulation?

\paragraph{Experiment design.} We implement 4 fusion mechanisms within the same PatchTransformer backbone (no shortcut): Cross-Attention Alignment, Gating, Text-as-Variable, and Output Fusion. For each mechanism and dataset, we compare performance with real context vs.\ zeroed context (identical architecture, only the input differs, 3 seeds). Note that MoME is not included in this comparison because it uses a fundamentally different backbone (14.3B-parameter LLM-based MoE) that cannot be reduced to our testbed framework; MoME's results are reported separately in Table~\ref{tab:mome}.

\begin{figure}[t]
    \centering
    \includegraphics[width=0.9\columnwidth]{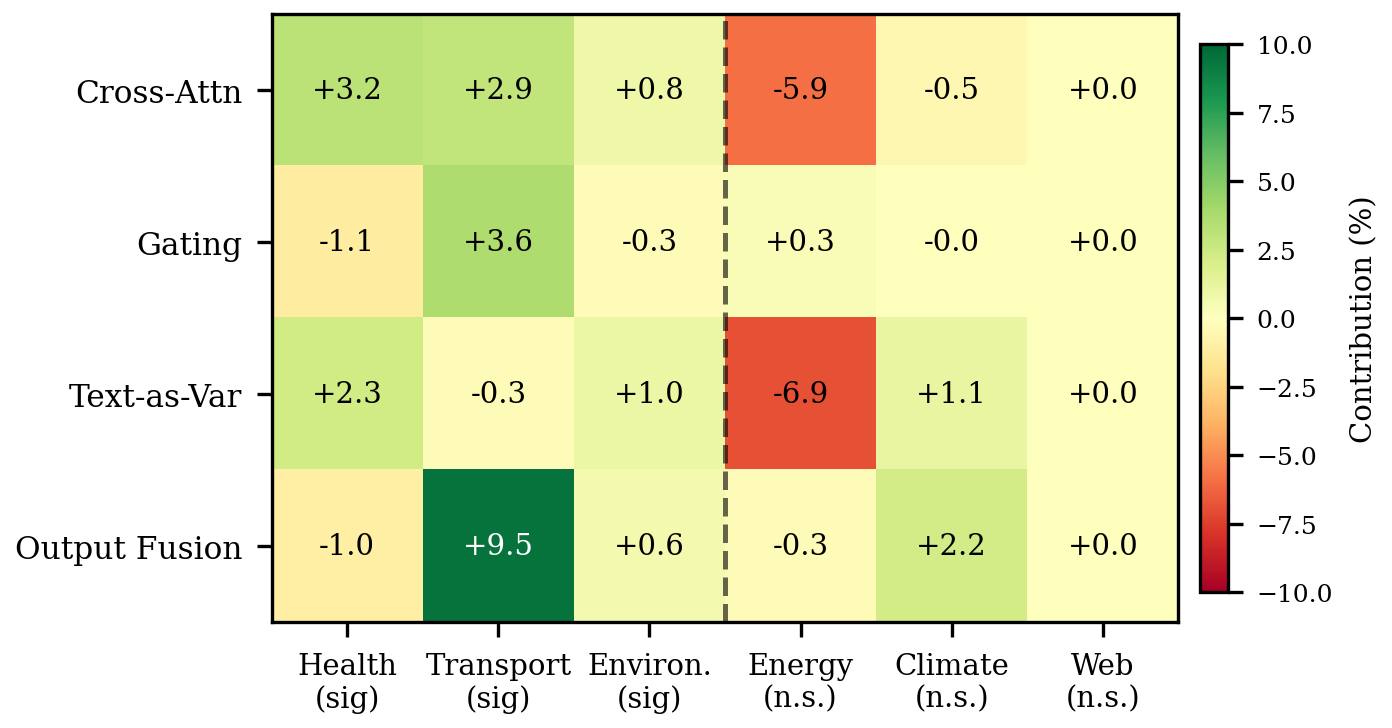}
    \caption{Context contribution (\%) across 4 fusion mechanisms and 5 datasets (testbed, no shortcut). Cells marked $\dagger$ (Cross-Attention and Text-as-Variable on Energy) use 10-seed evaluation; all others use 3-seed. Full per-cell std and significance tests are in Table~\ref{tab:app_multimodel}. On MI-significant datasets (left of dashed line), some mechanisms show small positive contributions; on MI-non-significant datasets (right), no mechanism shows a positive contribution. MoME is not shown here as it uses a different backbone (see Table~\ref{tab:mome}).}
    \label{fig:multimodel}
\end{figure}

\paragraph{Result.} Figure~\ref{fig:multimodel} shows that in our testbed, most context contributions are small. On MI-significant datasets, Cross-Attention (+3.2\% Health, +2.9\% Transport), Gating (+3.6\% Transport), and Output Fusion (+9.5\% Transport) show positive contributions. On MI-non-significant datasets, Gating and Output Fusion show $\approx$0\%, while Cross-Attention and Text-as-Variable show small \emph{negative} effects on Energy ($-$5.9\% and $-$6.9\%, 10-seed; neither significant at $\alpha=0.05$): attending to uninformative context introduces noise, and Cross-Attention is particularly vulnerable because it allocates capacity to context regardless of informativeness. Crucially, no MI-non-significant dataset shows a \emph{positive} context contribution---the MI test correctly identifies the ``no benefit'' regime.

\paragraph{Interpretation.} The testbed effects are small in magnitude compared to MoME's 29--51\%. As Finding 5 will show, context-target relationships are fundamentally nonlinear; our 2-layer testbed lacks the capacity to exploit them. The testbed's value is in confirming two patterns: (i) MI-non-significant datasets never show positive context benefit, and (ii) some fusion mechanisms (Cross-Attention) are vulnerable to degradation from uninformative context, a practical consideration for architecture selection. In summary: \textbf{the MI test's negative predictions (no benefit from context) hold across all tested fusion mechanisms}.

\subsection{Finding 4: A Calibrated Pre-Training Diagnostic}
\label{sec:finding4}

Based on Findings 1--3, we propose a calibrated diagnostic that practitioners can run \emph{before training any model} to decide whether multi-modal fusion is worth pursuing.

\paragraph{The diagnostic.} Given time series $X$, context $C$, and forecast horizon $h$:
\begin{enumerate}
    \item \textbf{Shortcut check} ($<$1 second): Compute $\rho_h$. If $\rho_h > 0.95$, output \textsc{SKIP\_FUSION}, since a last-value shortcut will dominate.
    \item \textbf{MI permutation test} ($\sim$1 minute): Estimate $\delta$ and compute $p$-value via 200 permutations.
    \begin{itemize}
        \item $p < 0.05$ $\rightarrow$ \textsc{TRY\_FUSION}: context is informative.
        \item $p \geq 0.05$ and power $\geq 0.8$ $\rightarrow$ \textsc{SKIP\_FUSION}: context is uninformative with high confidence.
        \item $p \geq 0.05$ and power $< 0.8$ $\rightarrow$ \textsc{INCONCLUSIVE}: insufficient evidence; recommend training a cheap model and validating on held-out data.
    \end{itemize}
\end{enumerate}

Power is classified as \textsc{HIGH} when $n \gtrsim 500$ with unique context per sample (unique-context fraction $> 0.5$), and \textsc{LOW} otherwise (substantially smaller $n$, or context embeddings reused cyclically). HealthUS ($n=491$) is on the boundary; we treat it as HIGH because its unique-context fraction is $>$0.95 and the MI test rejected the null at $p<0.001$. Figure~\ref{fig:flowchart} summarizes the decision logic.

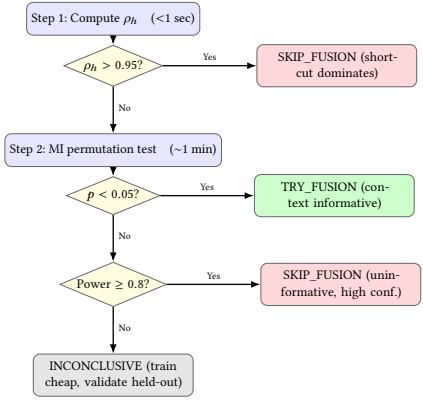
\begin{figure}[t]
    \centering
    \resizebox{0.65\columnwidth}{!}{%
    \begin{tikzpicture}[node distance=0.2cm and 0.4cm]
        \node[stepbox] (s1) {Step 1: Compute $\rho_h$\quad($<$1 sec)};
        \node[decision, below=of s1] (d1) {$\rho_h > 0.95$?};
        \node[skipterm, right=2.0cm of d1] (skip1) {SKIP\_FUSION (shortcut dominates)};

        \node[stepbox, below=1.0cm of d1] (s2) {Step 2: MI permutation test\quad($\sim$1 min)};
        \node[decision, below=of s2] (d2) {$p < 0.05$?};
        \node[tryterm, right=2.0cm of d2] (try) {TRY\_FUSION (context informative)};

        \node[decision, below=1.0cm of d2] (d3) {Power $\geq 0.8$?};
        \node[skipterm, right=2.0cm of d3] (skip2) {SKIP\_FUSION (uninformative, high conf.)};

        \node[inconterm, below=1.0cm of d3] (incon) {INCONCLUSIVE (train cheap, validate held-out)};

        \draw[flowarrow] (s1) -- (d1);
        \draw[flowarrow] (d1) -- node[above, font=\scriptsize] {Yes} (skip1);
        \draw[flowarrow] (d1) -- node[right, font=\scriptsize, pos=0.45] {No} (s2);
        \draw[flowarrow] (s2) -- (d2);
        \draw[flowarrow] (d2) -- node[above, font=\scriptsize] {Yes} (try);
        \draw[flowarrow] (d2) -- node[right, font=\scriptsize, pos=0.45] {No} (d3);
        \draw[flowarrow] (d3) -- node[above, font=\scriptsize] {Yes} (skip2);
        \draw[flowarrow] (d3) -- node[right, font=\scriptsize, pos=0.45] {No} (incon);
    \end{tikzpicture}}
    \caption{Practitioner decision framework. The diagnostic outputs one of three recommendations based on autocorrelation $\rho_h$, MI significance, and statistical power. Total computation time: $\sim$1 minute, zero GPU required.}
    \label{fig:flowchart}
\end{figure}

\paragraph{Validation.} Table~\ref{tab:diagnostic} validates the diagnostic on all 8 datasets. Key properties:
\begin{itemize}[leftmargin=*]
    \item \textbf{No false positives observed} in well-powered settings: on the datasets we test, every SKIP\_FUSION decision corresponds to no context-attributable benefit (we report this as an empirical observation, not a guarantee).
    \item \textbf{No misleading SKIPs}: SocialGood and Environment (which have positive MoME benefits but non-significant MI under the original embeddings) correctly receive INCONCLUSIVE rather than SKIP, because their power is LOW.
    \item \textbf{Correct TRY recommendations}: HealthUS (significant MI) correctly receives TRY, corresponding to +51\% actual benefit.
\end{itemize}

\begin{table}[t]
\centering
\caption{Diagnostic validation across all datasets. The diagnostic produces zero errors in the TRY/SKIP categories. INCONCLUSIVE is output when power is insufficient, correctly avoiding false negatives.}
\label{tab:diagnostic}
\small
\begin{tabular}{lccccl}
\toprule
Dataset & $n$ & $p$ & Power & Output & Actual \\
\midrule
HealthUS & 491 & $<$.001 & High & TRY & +51\% \checkmark \\
Transport & 5000 & .015 & High & TRY & +2\% \checkmark \\
\midrule
Energy & 1059 & .580 & High & SKIP & $\approx$0$^\sharp$ \checkmark \\
Climate & 1200 & .705 & High & SKIP & $-$0.5\% \checkmark \\
Web & 1068 & 1.0 & High & SKIP & $-$0.4\% \checkmark \\
FinMultiTime & 2000 & n/s & High & SKIP & +6.7\%$^\sharp$ \checkmark \\
\midrule
SocialGood & 363 & .145 & Low & INCON. & +32\% \checkmark \\
Environment & 1597$^*$ & .615 & Low & INCON. & +41.5\% \checkmark \\
\bottomrule
\end{tabular}

\smallskip
\noindent\scriptsize $^*$156 unique embeddings reused cyclically. INCON.\ = INCONCLUSIVE (no claim made; recommend trying). ``Actual'' = MoME routing contribution (testbed for Transport). $^\sharp$High-$\rho$ SKIP datasets: the MoME routing contribution sits at the modulation pathway's capacity floor and carries no MI signal (Energy not statistically distinguishable from zero; see Table~\ref{tab:mome} footnote).
\end{table}

\subsection{Finding 5: Context-Target Relationships are Nonlinear}
\label{sec:nonlinear}

A natural question is whether a simpler diagnostic, such as cross-validated linear regression, could replace the MI permutation test. We test this by fitting a Ridge regression model to predict the target from context (5-fold CV) and measuring whether adding context improves prediction over history alone.

\paragraph{Result.} Cross-validated linear prediction gain from context is \emph{zero or negative} on \emph{all} datasets, including HealthUS, where MoME achieves +51\%. This reveals that context-target relationships are fundamentally nonlinear: a linear model cannot detect or exploit the signal that a 14.3B-parameter model exploits for a 51\% improvement.

\paragraph{Implications.}
\begin{enumerate}[leftmargin=*]
    \item The MI permutation test (k-NN based, nonparametric) is the correct diagnostic precisely because it captures nonlinear dependencies that linear methods miss.
    \item The magnitude gap between our testbed (+2--5\%) and MoME (+29--51\%) reflects \emph{model capacity}, not experimental noise. A 2-layer Transformer cannot exploit the same nonlinear signal that a 14.3B-parameter LLM-based model can.
    \item Practitioners should not use linear probes or simple correlation tests to assess context value; these will systematically underestimate the opportunity.
\end{enumerate}

\subsection{Finding 6: Large-Scale Validation on 27 Monash Datasets}
\label{sec:finding6}

To validate the $\rho$-based component of our diagnostic beyond our 8 datasets, we compute autocorrelation at the prediction horizon for 27 datasets from the Monash Time Series Forecasting Archive~\cite{godahewa2021monash}, spanning hourly to yearly frequencies. For each dataset, we measure the relative performance of the repeat-last-value shortcut against a naive mean predictor.

\begin{figure}[t]
    \centering
    \includegraphics[width=0.85\columnwidth]{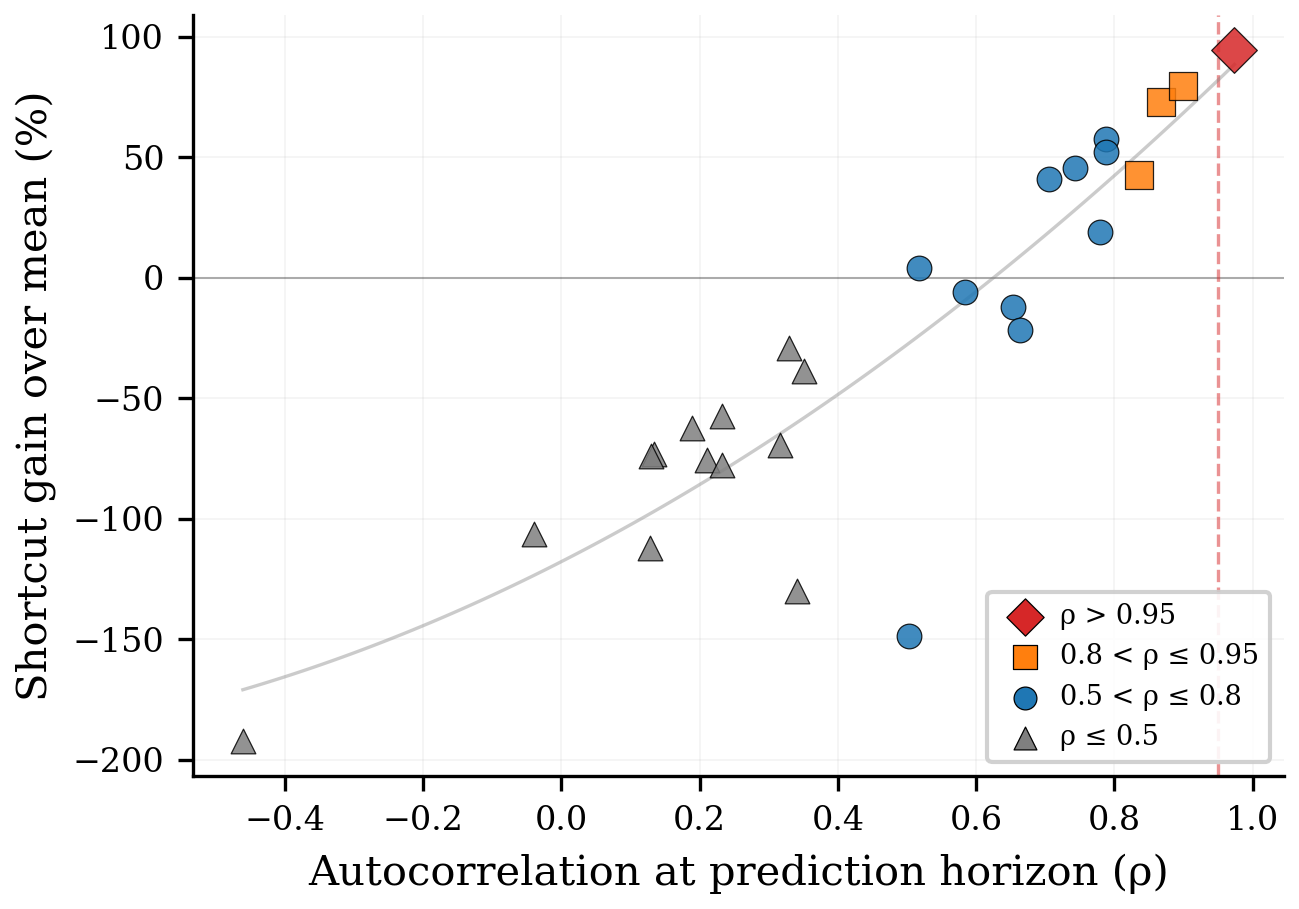}
    \caption{Large-scale validation on 27 Monash Archive datasets. Autocorrelation at the prediction horizon strongly predicts shortcut dominance (Spearman $r = 0.888$, $p < 0.0001$). Datasets with $\rho > 0.95$ (red diamonds) show the strongest shortcut performance, validating the SKIP\_FUSION threshold.}
    \label{fig:monash}
\end{figure}

\paragraph{Result.} Figure~\ref{fig:monash} shows a strong monotonic relationship: autocorrelation at the prediction horizon strongly predicts shortcut dominance (Spearman $r = 0.888$, $p < 0.0001$, $n = 27$). All datasets with $\rho > 0.8$ show positive shortcut gain (42--95\%), while all datasets with $\rho < 0.3$ show negative shortcut gain (the shortcut is worse than the mean). This validates our diagnostic's first step (the $\rho > 0.95$ threshold for SKIP\_FUSION) at scale across diverse domains and frequencies.

\paragraph{Scope of validation.} Monash datasets lack auxiliary context, so this validates only the $\rho$-based component; the MI component remains validated on our 8 text-context datasets.

\section{Discussion}

\paragraph{The shortcut tradeoff for practitioners.} Our results reveal a practical tension. On high-$\rho$ data (e.g., FinMultiTime, $\rho = 0.999$), the naive repeat-last-value baseline \emph{outperforms} even a 14.3B-parameter model, so practitioners should use simple baselines (DLinear, repeat-last-value) and skip multi-modal fusion entirely. On moderate-$\rho$ data with informative context (e.g., HealthUS, $\rho = 0.77$), MoME \emph{with} text modulation (MSE 0.399) edges out the naive baseline (0.409) and is far better than MoME \emph{without} text (0.817); the value of fusion here is precisely the text pathway, not the architecture. The diagnostic identifies which regime applies before any model is trained.

\paragraph{INCONCLUSIVE: frequency and impact.} 2 of 8 datasets receive INCONCLUSIVE, both due to low MI-test power (small $n$ or reused embeddings), not methodology failure. The asymmetric cost structure justifies this: a false negative costs one verification run, whereas a false positive (SKIP where fusion helps) costs accuracy with no recovery path.

\paragraph{Metric comparability.} Environment uses MAE while other MoME datasets use MSE, because MoME's Environment task outputs a different prediction format. The routing contribution percentages (computed as relative improvement within each dataset) are comparable across metrics: both measure ``how much does modulation reduce error relative to no-modulation.'' However, absolute MSE/MAE values should not be compared across datasets.

\paragraph{Comparison with Zhang et al.\ (2025).} Zhang et al.~\cite{zhang2025when} find multi-modal benefits ``condition-dependent'' but do not explain \emph{why}; our framework does. On FinMultiTime they would see MoME's small +6.7\% and conclude ``fusion barely helps,'' whereas our $\rho{=}0.999$ analysis shows a trivial baseline already beats MoME---no context can improve on it. On Web they would see multi-modal methods beat TS-only baselines by 50\%+ and conclude ``fusion helps,'' whereas our zero-context ablation reveals an architectural confound, not context exploitation.

\paragraph{RBU as interpretive framework.} RBU (Proposition~\ref{prop:rbu}) explains \emph{why} the pattern exists: $(1-\rho_h^2)$ is the room left by the shortcut, and $(1-2^{-2\delta})$ is the fraction context can fill. It is exact only under joint Gaussianity (Appendix~\ref{app:rbu_proof}); we use the plug-in $\widehat{\text{RBU}}$ as an interpretive sanity check, not a bound, since its Kraskov $\hat\delta$ under-detects nonlinear dependence and so biases $\widehat{\text{RBU}}$ low (consistent with HealthUS, where MoME realizes a large 51\% benefit). The distribution-free shortcut ceiling (Fact~\ref{fact:ceiling}) remains the only strict bound.

\paragraph{Foundation model reference.} Chronos~\cite{ansari2024chronos} (zero-shot, no context) outperforms the shortcut on Health (+14\%), Transport (+12\%), and Web (+23\%), but underperforms on Energy ($-$8\%). This is consistent with our framework: on high-$\rho$ data (Energy, $\rho = 0.99$), even foundation models cannot beat simple shortcuts.

\paragraph{Limitations.}
\begin{itemize}[leftmargin=*]
    \item \textbf{Positive magnitudes concentrated in MoME.} The large routing contributions (+29--51\%) come from a single model family. The testbed corroborates the \emph{direction} (sign-consistent, +2--5\%) but not the magnitude; the gap is a capacity/overfitting effect (a 2-layer from-scratch model cannot exploit the high-dimensional signal a pretrained 14.3B backbone can). Practitioners can rely on the \emph{sign} of our diagnostic, not the specific percentages.
    \item \textbf{Single-lag MI conditioning.} Our residual removes only the first-order dependence on $X_t$ (Section~\ref{sec:theory}); on strongly seasonal data the TRY arm may overstate context value. The SKIP arm and the distribution-free null are unaffected.
    \item \textbf{Modality and dataset scope.} Context comes from Time-MMD (text) and FinMultiTime (financial news); we do not test tabular, image, or audio context. The diagnostic is modality-agnostic by construction, but its empirical validation here is confined to these sources.
    \item \textbf{Probabilistic forecasting.} MoME outputs point forecasts only; evaluation with probabilistic metrics (wQL, CRPS) would strengthen the assessment.
    \item \textbf{Coverage gaps.} The MoME shortcut experiment could not run on Energy and Environment (GPU memory), covering 3 of 5 MI-significant datasets; the Monash validation tests only the $\rho$-based component (no text context); and TimesFM~\cite{das2024timesfm}/Moirai~\cite{woo2024moirai} could not be installed, leaving Chronos~\cite{ansari2024chronos} as the foundation-model reference.
    \item \textbf{Claim scope.} ``No false positives in well-powered settings'' is an empirical observation on the datasets we test, not a proven general property.
\end{itemize}

\section{Related Work}
\label{sec:related}

\paragraph{Time series forecasting architectures.} Modern forecasting has evolved from classical approaches~\cite{makridakis2018statistical} through deep Transformer-based methods that handle long horizons, periodicity, and distribution shift~\cite{zhou2021informer,wu2021autoformer,nie2023patchtst,liu2024itransformer,kim2022reversible}. A counter-trend questions whether such complexity is necessary: DLinear~\cite{zeng2023dlinear} showed that single-layer linear models match Transformers on many benchmarks. Our work continues this skeptical line: we identify a structural reason (last-value shortcuts capturing most of the predictable variance) that explains why simple baselines remain competitive on autocorrelated data, and shows the same factor governs whether multi-modal fusion provides genuine benefit.

\paragraph{Time series foundation models.} Pretraining on large time series corpora has produced general-purpose forecasters: Chronos~\cite{ansari2024chronos}, TimesFM~\cite{das2024timesfm}, Moirai~\cite{woo2024moirai}, and MOMENT~\cite{goswami2024moment} all demonstrate strong zero-shot transfer. A parallel line repurposes pretrained language models for time series~\cite{zhou2023onefitsall,jin2024timellm}. None of these foundation models incorporate auxiliary context, which leaves open the question we address: under what conditions can context provide value beyond what these models already extract from history?

\paragraph{Multi-modal time series forecasting.} The field has grown rapidly, spanning output-level fusion~\cite{liu2024timemmd,lee2024moat}, cross-attention~\cite{liu2025timecma}, text-as-variable~\cite{li2025tats}, gating~\cite{survey2025mmts,perez2018film}, and mixture-of-experts~\cite{mome2025,shazeer2017outrageously,fedus2022switch,segmoe2025}. A recent survey~\cite{survey2025mmts} catalogs these mechanisms, and broader multi-modal learning has been surveyed in~\cite{baltrusaitis2019multimodal}. Our work does not propose a new fusion method; instead, we characterize when existing methods provide genuine value versus architectural confounds.

\paragraph{Strong simple baselines as a benchmark for ``does it help?''} DLinear~\cite{zeng2023dlinear} showed linear models match Transformers; RevIN~\cite{kim2022reversible} showed normalization outperforms complex architectures. Our results extend this line by demonstrating that last-value shortcuts similarly dominate multi-modal fusion on autocorrelated data, and that the dominance is \emph{causal} in the interventionist sense~\cite{pearl2009causality}: adding a shortcut to MoME suppresses routing contribution, and corrupting context degrades it monotonically. The pattern is consistent with the broader phenomenon of \emph{shortcut learning}~\cite{geirhos2020shortcut}, where models exploit easy predictive signals at the cost of harder, more transferable ones.

\paragraph{Relation to recent strong uni-modal forecasters.} A parallel line pushes uni-modal accuracy through specialized architectures: multi-period decomposition (MLF~\cite{zhang2025mlf}), variable/time-aware hyper-states (TimePro~\cite{ma2025timepro}), semantics-enhanced MLP-mixing (SEMixer~\cite{zhang2026semixer}), and pattern-specific experts under patch-level distribution shift~\cite{sun2024patchexpert}. These are \emph{orthogonal} to us: they improve how a model extracts signal from history alone, whereas we ask whether \emph{auxiliary context} adds value beyond history, and provide a diagnostic rather than a forecaster. A stronger uni-modal backbone in fact reinforces our thesis---the better history is modeled, the less room remains for context (Fact~\ref{fact:ceiling}), making the shortcut-dominance condition \emph{more} binding. We therefore do not benchmark against them, as our claims concern the conditions for context value, not SOTA accuracy.

\paragraph{When does multimodality help?} Concurrent to our work, Zhang et al.~\cite{zhang2025when} empirically evaluate when multi-modal fusion helps, finding benefits are condition-dependent. Our analysis goes beyond theirs in three ways. First, we provide a mechanistic explanation centered on shortcut dominance and MI significance, with a Gaussian-case quantification (RBU). Second, we offer causal demonstrations via shortcut addition and context degradation rather than purely observational evidence. Third, the diagnostic we propose includes explicit power characterization and a ternary output that handles low-power cases honestly, instead of forcing a binary recommendation.

\paragraph{Mutual information estimation.} We use the Kraskov k-NN estimator~\cite{kraskov2004estimating} with permutation testing. Power limitations of nonparametric MI testing are characterized by Berrett \& Samworth~\cite{berrett2019nonparametric}. Recent neural MI estimators such as MINE~\cite{belghazi2018mine} and variational bounds~\cite{poole2019variational} offer alternative routes; we choose Kraskov for its established theoretical properties and ease of integration with the permutation procedure. Sparse attention mechanisms in our routing testbed build on entmax~\cite{peters2019sparse,correia2019adaptively}.

\section{Conclusion}

We identify two conditions that must both hold for auxiliary context to help in time series forecasting: (1) absence of a competing shortcut, and (2) statistically informative context. We validate these conditions causally: adding a shortcut to MoME suppresses routing contribution by 77--93\% across 3 datasets (HealthUS: 62.9\% $\to$ 14.1\%, SocialGood: 34.3\% $\to$ 3.3\%, HealthAFR: 24.2\% $\to$ 1.7\%); corrupting context drives the context-specific benefit from +44\% (clean text) to negative (partially corrupted text). The $\rho$-based diagnostic is validated at scale on 27 Monash Archive datasets (Spearman $r = 0.888$, $p < 0.0001$). We are explicit that the large positive magnitudes come from a single model family (MoME); the testbed corroborates direction, not size.

We provide a calibrated diagnostic (TRY/SKIP/INCONCLUSIVE) that, on the datasets we test, yields no false positives in well-powered settings, enabling practitioners to decide whether to invest in fusion \emph{before training a single model}. The practical takeaway: compute $\rho_h$ and run a 1-minute MI permutation test. If the shortcut dominates or context is uninformative, skip complex fusion, since no amount of architectural sophistication will help.

\bibliographystyle{ACM-Reference-Format}
\bibliography{references_analysis}


\begin{thebibliography}{40}


\ifx \showCODEN    \undefined \def \showCODEN     #1{\unskip}     \fi
\ifx \showISBNx    \undefined \def \showISBNx     #1{\unskip}     \fi
\ifx \showISBNxiii \undefined \def \showISBNxiii  #1{\unskip}     \fi
\ifx \showISSN     \undefined \def \showISSN      #1{\unskip}     \fi
\ifx \showLCCN     \undefined \def \showLCCN      #1{\unskip}     \fi
\ifx \shownote     \undefined \def \shownote      #1{#1}          \fi
\ifx \showarticletitle \undefined \def \showarticletitle #1{#1}   \fi
\ifx \showURL      \undefined \def \showURL       {\relax}        \fi
\providecommand\bibfield[2]{#2}
\providecommand\bibinfo[2]{#2}
\providecommand\natexlab[1]{#1}
\providecommand\showeprint[2][]{arXiv:#2}

\bibitem[Ansari et~al\mbox{.}(2024)]%
        {ansari2024chronos}
\bibfield{author}{\bibinfo{person}{Abdul~Fatir Ansari},
  \bibinfo{person}{Lorenzo Stella}, \bibinfo{person}{Caner Turkmen},
  \bibinfo{person}{Xiyuan Zhang}, \bibinfo{person}{Pedro Mercado},
  \bibinfo{person}{Huibin Shen}, \bibinfo{person}{Michael Bohlke-Schneider},
  \bibinfo{person}{Yuyang Wang}, {et~al\mbox{.}}}
  \bibinfo{year}{2024}\natexlab{}.
\newblock \showarticletitle{Chronos: Learning the Language of Time Series}.
\newblock \bibinfo{journal}{\emph{arXiv preprint arXiv:2403.07815}}
  (\bibinfo{year}{2024}).
\newblock


\bibitem[Baltrusaitis et~al\mbox{.}(2019)]%
        {baltrusaitis2019multimodal}
\bibfield{author}{\bibinfo{person}{Tadas Baltrusaitis},
  \bibinfo{person}{Chaitanya Ahuja}, {and} \bibinfo{person}{Louis-Philippe
  Morency}.} \bibinfo{year}{2019}\natexlab{}.
\newblock \showarticletitle{Multimodal Machine Learning: A Survey and
  Taxonomy}.
\newblock \bibinfo{journal}{\emph{IEEE Transactions on Pattern Analysis and
  Machine Intelligence}}  \bibinfo{volume}{41} (\bibinfo{year}{2019}),
  \bibinfo{pages}{423--443}.
\newblock


\bibitem[Belghazi et~al\mbox{.}(2018)]%
        {belghazi2018mine}
\bibfield{author}{\bibinfo{person}{Mohamed~Ishmael Belghazi},
  \bibinfo{person}{Aristide Baratin}, \bibinfo{person}{Sai Rajeswar},
  \bibinfo{person}{Sherjil Ozair}, \bibinfo{person}{Yoshua Bengio},
  \bibinfo{person}{Aaron Courville}, {and} \bibinfo{person}{R~Devon Hjelm}.}
  \bibinfo{year}{2018}\natexlab{}.
\newblock \showarticletitle{Mutual Information Neural Estimation}. In
  \bibinfo{booktitle}{\emph{ICML}}.
\newblock


\bibitem[Berrett and Samworth(2019)]%
        {berrett2019nonparametric}
\bibfield{author}{\bibinfo{person}{Thomas~B Berrett} {and}
  \bibinfo{person}{Richard~J Samworth}.} \bibinfo{year}{2019}\natexlab{}.
\newblock \showarticletitle{Nonparametric independence testing via mutual
  information}.
\newblock \bibinfo{journal}{\emph{Biometrika}} \bibinfo{volume}{106},
  \bibinfo{number}{3} (\bibinfo{year}{2019}), \bibinfo{pages}{547--566}.
\newblock


\bibitem[Correia et~al\mbox{.}(2019)]%
        {correia2019adaptively}
\bibfield{author}{\bibinfo{person}{Gon{\c{c}}alo~M Correia},
  \bibinfo{person}{Vlad Niculae}, {and} \bibinfo{person}{Andr{\'e}~FT
  Martins}.} \bibinfo{year}{2019}\natexlab{}.
\newblock \showarticletitle{Adaptively Sparse Transformers}. In
  \bibinfo{booktitle}{\emph{EMNLP}}.
\newblock


\bibitem[Cover and Thomas(2006)]%
        {cover2006elements}
\bibfield{author}{\bibinfo{person}{Thomas~M Cover} {and} \bibinfo{person}{Joy~A
  Thomas}.} \bibinfo{year}{2006}\natexlab{}.
\newblock \bibinfo{booktitle}{\emph{Elements of Information Theory}
  (\bibinfo{edition}{2nd} ed.)}.
\newblock \bibinfo{publisher}{Wiley-Interscience}.
\newblock


\bibitem[Das et~al\mbox{.}(2024)]%
        {das2024timesfm}
\bibfield{author}{\bibinfo{person}{Abhimanyu Das} {et~al\mbox{.}}}
  \bibinfo{year}{2024}\natexlab{}.
\newblock \showarticletitle{A Decoder-Only Foundation Model for Time-Series
  Forecasting}. In \bibinfo{booktitle}{\emph{ICML}}.
\newblock


\bibitem[Fedus et~al\mbox{.}(2022)]%
        {fedus2022switch}
\bibfield{author}{\bibinfo{person}{William Fedus}, \bibinfo{person}{Barret
  Zoph}, {and} \bibinfo{person}{Noam Shazeer}.}
  \bibinfo{year}{2022}\natexlab{}.
\newblock \showarticletitle{Switch Transformers: Scaling to Trillion Parameter
  Models with Simple and Efficient Sparsity}.
\newblock \bibinfo{journal}{\emph{JMLR}}  \bibinfo{volume}{23}
  (\bibinfo{year}{2022}).
\newblock


\bibitem[Geirhos et~al\mbox{.}(2020)]%
        {geirhos2020shortcut}
\bibfield{author}{\bibinfo{person}{Robert Geirhos} {et~al\mbox{.}}}
  \bibinfo{year}{2020}\natexlab{}.
\newblock \showarticletitle{Shortcut Learning in Deep Neural Networks}.
\newblock \bibinfo{journal}{\emph{Nature Machine Intelligence}}
  \bibinfo{volume}{2} (\bibinfo{year}{2020}), \bibinfo{pages}{665--673}.
\newblock


\bibitem[Godahewa et~al\mbox{.}(2021)]%
        {godahewa2021monash}
\bibfield{author}{\bibinfo{person}{Rakshitha Godahewa},
  \bibinfo{person}{Christoph Bergmeir}, \bibinfo{person}{Geoffrey~I Webb},
  \bibinfo{person}{Rob~J Hyndman}, {and} \bibinfo{person}{Pablo
  Montero-Manso}.} \bibinfo{year}{2021}\natexlab{}.
\newblock \showarticletitle{Monash Time Series Forecasting Archive}.
\newblock \bibinfo{journal}{\emph{Neural Information Processing Systems Track
  on Datasets and Benchmarks}} (\bibinfo{year}{2021}).
\newblock


\bibitem[Goswami et~al\mbox{.}(2024)]%
        {goswami2024moment}
\bibfield{author}{\bibinfo{person}{Mononito Goswami} {et~al\mbox{.}}}
  \bibinfo{year}{2024}\natexlab{}.
\newblock \showarticletitle{MOMENT: A Family of Open Time-Series Foundation
  Models}. In \bibinfo{booktitle}{\emph{ICML}}.
\newblock


\bibitem[Jiang et~al\mbox{.}(2025)]%
        {survey2025mmts}
\bibfield{author}{\bibinfo{person}{Yushan Jiang}, \bibinfo{person}{Kanghui
  Ning}, \bibinfo{person}{Zijie Pan}, \bibinfo{person}{Xuyang Shen},
  \bibinfo{person}{Jingchao Ni}, \bibinfo{person}{Wenchao Yu},
  \bibinfo{person}{Anderson Schneider}, \bibinfo{person}{Haifeng Chen},
  \bibinfo{person}{Yuriy Nevmyvaka}, {and} \bibinfo{person}{Dongjin Song}.}
  \bibinfo{year}{2025}\natexlab{}.
\newblock \showarticletitle{Multi-modal Time Series Analysis: A Tutorial and
  Survey}.
\newblock \bibinfo{journal}{\emph{arXiv preprint arXiv:2503.13709}}
  (\bibinfo{year}{2025}).
\newblock


\bibitem[Jin et~al\mbox{.}(2024)]%
        {jin2024timellm}
\bibfield{author}{\bibinfo{person}{Ming Jin} {et~al\mbox{.}}}
  \bibinfo{year}{2024}\natexlab{}.
\newblock \showarticletitle{Time-LLM: Time Series Forecasting by Reprogramming
  Large Language Models}. In \bibinfo{booktitle}{\emph{ICLR}}.
\newblock


\bibitem[Kim et~al\mbox{.}(2022)]%
        {kim2022reversible}
\bibfield{author}{\bibinfo{person}{Taesung Kim} {et~al\mbox{.}}}
  \bibinfo{year}{2022}\natexlab{}.
\newblock \showarticletitle{Reversible Instance Normalization for Accurate
  Time-Series Forecasting against Distribution Shift}. In
  \bibinfo{booktitle}{\emph{ICLR}}.
\newblock


\bibitem[Kraskov et~al\mbox{.}(2004)]%
        {kraskov2004estimating}
\bibfield{author}{\bibinfo{person}{Alexander Kraskov}, \bibinfo{person}{Harald
  St\"ogbauer}, {and} \bibinfo{person}{Peter Grassberger}.}
  \bibinfo{year}{2004}\natexlab{}.
\newblock \showarticletitle{Estimating Mutual Information}.
\newblock \bibinfo{journal}{\emph{Physical Review E}} \bibinfo{volume}{69},
  \bibinfo{number}{6} (\bibinfo{year}{2004}).
\newblock


\bibitem[Lee et~al\mbox{.}(2024)]%
        {lee2024moat}
\bibfield{author}{\bibinfo{person}{Geon Lee}, \bibinfo{person}{Wenchao Yu},
  \bibinfo{person}{Wei Cheng}, {and} \bibinfo{person}{Haifeng Chen}.}
  \bibinfo{year}{2024}\natexlab{}.
\newblock \bibinfo{title}{MoAT: Multi-Modal Augmented Time Series Forecasting}.
\newblock
\urldef\tempurl%
\url{https://openreview.net/forum?id=uRXxnoqDHH}
\showURL{%
\tempurl}
\newblock
\shownote{OpenReview preprint}.


\bibitem[Li et~al\mbox{.}(2026)]%
        {li2025tats}
\bibfield{author}{\bibinfo{person}{Zihao Li}, \bibinfo{person}{Xiao Lin},
  \bibinfo{person}{Zhining Liu}, \bibinfo{person}{Jiaru Zou},
  \bibinfo{person}{Ziwei Wu}, \bibinfo{person}{Lecheng Zheng},
  \bibinfo{person}{Dongqi Fu}, \bibinfo{person}{Yada Zhu},
  \bibinfo{person}{Hendrik Hamann}, \bibinfo{person}{Hanghang Tong}, {and}
  \bibinfo{person}{Jingrui He}.} \bibinfo{year}{2026}\natexlab{}.
\newblock \showarticletitle{Language in the Flow of Time: Time-Series-Paired
  Texts Weaved into a Unified Temporal Narrative}. In
  \bibinfo{booktitle}{\emph{International Conference on Learning
  Representations (ICLR)}}.
\newblock


\bibitem[Liu et~al\mbox{.}(2025)]%
        {liu2025timecma}
\bibfield{author}{\bibinfo{person}{Chenxi Liu} {et~al\mbox{.}}}
  \bibinfo{year}{2025}\natexlab{}.
\newblock \showarticletitle{TimeCMA: Towards LLM-Empowered Multivariate Time
  Series Forecasting via Cross-Modality Alignment}. In
  \bibinfo{booktitle}{\emph{AAAI}}.
\newblock


\bibitem[Liu et~al\mbox{.}(2024a)]%
        {liu2024timemmd}
\bibfield{author}{\bibinfo{person}{Haoxin Liu} {et~al\mbox{.}}}
  \bibinfo{year}{2024}\natexlab{a}.
\newblock \showarticletitle{Time-MMD: Multi-Domain Multimodal Dataset for Time
  Series Analysis}. In \bibinfo{booktitle}{\emph{NeurIPS Datasets and
  Benchmarks}}.
\newblock


\bibitem[Liu et~al\mbox{.}(2024b)]%
        {liu2024itransformer}
\bibfield{author}{\bibinfo{person}{Yong Liu}, \bibinfo{person}{Tengge Hu},
  \bibinfo{person}{Haoran Zhang}, \bibinfo{person}{Haixu Wu},
  \bibinfo{person}{Shiyu Wang}, \bibinfo{person}{Lintao Ma}, {and}
  \bibinfo{person}{Mingsheng Long}.} \bibinfo{year}{2024}\natexlab{b}.
\newblock \showarticletitle{iTransformer: Inverted Transformers Are Effective
  for Time Series Forecasting}. In \bibinfo{booktitle}{\emph{ICLR}}.
\newblock


\bibitem[Ma et~al\mbox{.}(2025)]%
        {ma2025timepro}
\bibfield{author}{\bibinfo{person}{Xiaowen Ma}, \bibinfo{person}{Zhenliang Ni},
  \bibinfo{person}{Shuai Xiao}, {and} \bibinfo{person}{Xinghao Chen}.}
  \bibinfo{year}{2025}\natexlab{}.
\newblock \showarticletitle{{TimePro}: Efficient Multivariate Long-term Time
  Series Forecasting with Variable- and Time-Aware Hyper-state}. In
  \bibinfo{booktitle}{\emph{International Conference on Machine Learning
  (ICML)}}.
\newblock
\newblock
\shownote{arXiv:2505.20774}.


\bibitem[Makridakis et~al\mbox{.}(2018)]%
        {makridakis2018statistical}
\bibfield{author}{\bibinfo{person}{Spyros Makridakis},
  \bibinfo{person}{Evangelos Spiliotis}, {and} \bibinfo{person}{Vassilios
  Assimakopoulos}.} \bibinfo{year}{2018}\natexlab{}.
\newblock \showarticletitle{Statistical and Machine Learning Forecasting
  Methods: Concerns and Ways Forward}.
\newblock \bibinfo{journal}{\emph{PloS one}}  \bibinfo{volume}{13}
  (\bibinfo{year}{2018}).
\newblock


\bibitem[Nie et~al\mbox{.}(2023)]%
        {nie2023patchtst}
\bibfield{author}{\bibinfo{person}{Yuqi Nie}, \bibinfo{person}{Nam~H Nguyen},
  \bibinfo{person}{Phanwadee Sinthong}, {and} \bibinfo{person}{Jayant
  Kalagnanam}.} \bibinfo{year}{2023}\natexlab{}.
\newblock \showarticletitle{A Time Series is Worth 64 Words: Long-term
  Forecasting with Transformers}. In \bibinfo{booktitle}{\emph{ICLR}}.
\newblock


\bibitem[Ortigossa and Segal(2026)]%
        {segmoe2025}
\bibfield{author}{\bibinfo{person}{Evandro~S Ortigossa} {and}
  \bibinfo{person}{Eran Segal}.} \bibinfo{year}{2026}\natexlab{}.
\newblock \showarticletitle{Multi-Resolution Segment-wise Mixture-of-Experts
  for Time Series Forecasting Transformers}.
\newblock \bibinfo{journal}{\emph{arXiv preprint arXiv:2601.21641}}
  (\bibinfo{year}{2026}).
\newblock


\bibitem[Pearl(2009)]%
        {pearl2009causality}
\bibfield{author}{\bibinfo{person}{Judea Pearl}.}
  \bibinfo{year}{2009}\natexlab{}.
\newblock \bibinfo{booktitle}{\emph{Causality} (\bibinfo{edition}{2nd} ed.)}.
\newblock \bibinfo{publisher}{Cambridge University Press}.
\newblock


\bibitem[Perez et~al\mbox{.}(2018)]%
        {perez2018film}
\bibfield{author}{\bibinfo{person}{Ethan Perez}, \bibinfo{person}{Florian
  Strub}, \bibinfo{person}{Harm De~Vries}, \bibinfo{person}{Vincent Dumoulin},
  {and} \bibinfo{person}{Aaron Courville}.} \bibinfo{year}{2018}\natexlab{}.
\newblock \showarticletitle{FiLM: Visual Reasoning with a General Conditioning
  Layer}. In \bibinfo{booktitle}{\emph{AAAI}}.
\newblock


\bibitem[Peters et~al\mbox{.}(2019)]%
        {peters2019sparse}
\bibfield{author}{\bibinfo{person}{Ben Peters}, \bibinfo{person}{Vlad Niculae},
  {and} \bibinfo{person}{Andr{\'e}~FT Martins}.}
  \bibinfo{year}{2019}\natexlab{}.
\newblock \showarticletitle{Sparse Sequence-to-Sequence Models}. In
  \bibinfo{booktitle}{\emph{ACL}}.
\newblock


\bibitem[Poole et~al\mbox{.}(2019)]%
        {poole2019variational}
\bibfield{author}{\bibinfo{person}{Ben Poole}, \bibinfo{person}{Sherjil Ozair},
  \bibinfo{person}{Aaron Van Den~Oord}, \bibinfo{person}{Alex Alemi}, {and}
  \bibinfo{person}{George Tucker}.} \bibinfo{year}{2019}\natexlab{}.
\newblock \showarticletitle{On Variational Bounds of Mutual Information}. In
  \bibinfo{booktitle}{\emph{ICML}}.
\newblock


\bibitem[Shazeer et~al\mbox{.}(2017)]%
        {shazeer2017outrageously}
\bibfield{author}{\bibinfo{person}{Noam Shazeer} {et~al\mbox{.}}}
  \bibinfo{year}{2017}\natexlab{}.
\newblock \showarticletitle{Outrageously Large Neural Networks: The
  Sparsely-Gated Mixture-of-Experts Layer}. In
  \bibinfo{booktitle}{\emph{ICLR}}.
\newblock


\bibitem[Sun et~al\mbox{.}(2025)]%
        {sun2024patchexpert}
\bibfield{author}{\bibinfo{person}{Yanru Sun}, \bibinfo{person}{Zongxia Xie},
  \bibinfo{person}{Emadeldeen Eldele}, \bibinfo{person}{Dongyue Chen},
  \bibinfo{person}{Qinghua Hu}, {and} \bibinfo{person}{Min Wu}.}
  \bibinfo{year}{2025}\natexlab{}.
\newblock \showarticletitle{Learning Pattern-Specific Experts for Time Series
  Forecasting Under Patch-level Distribution Shift}. In
  \bibinfo{booktitle}{\emph{Advances in Neural Information Processing Systems
  (NeurIPS)}}.
\newblock
\newblock
\shownote{arXiv:2410.09836}.


\bibitem[Woo et~al\mbox{.}(2024)]%
        {woo2024moirai}
\bibfield{author}{\bibinfo{person}{Gerald Woo} {et~al\mbox{.}}}
  \bibinfo{year}{2024}\natexlab{}.
\newblock \showarticletitle{Unified Training of Universal Time Series
  Forecasting Transformers}. In \bibinfo{booktitle}{\emph{ICML}}.
\newblock


\bibitem[Wu et~al\mbox{.}(2021)]%
        {wu2021autoformer}
\bibfield{author}{\bibinfo{person}{Haixu Wu} {et~al\mbox{.}}}
  \bibinfo{year}{2021}\natexlab{}.
\newblock \showarticletitle{Autoformer: Decomposition Transformers with
  Auto-Correlation for Long-Term Series Forecasting}. In
  \bibinfo{booktitle}{\emph{NeurIPS}}.
\newblock


\bibitem[Xu et~al\mbox{.}(2025)]%
        {chen2025finmultitime}
\bibfield{author}{\bibinfo{person}{Wenyan Xu}, \bibinfo{person}{Dawei Xiang},
  \bibinfo{person}{Yue Liu}, \bibinfo{person}{Xiyu Wang},
  \bibinfo{person}{Yanxiang Ma}, \bibinfo{person}{Liang Zhang},
  \bibinfo{person}{Shu Hu}, \bibinfo{person}{Chang Xu}, {and}
  \bibinfo{person}{Jiaheng Zhang}.} \bibinfo{year}{2025}\natexlab{}.
\newblock \showarticletitle{FinMultiTime: A Four-Modal Bilingual Dataset for
  Financial Time-Series Analysis}.
\newblock \bibinfo{journal}{\emph{arXiv preprint arXiv:2506.05019}}
  (\bibinfo{year}{2025}).
\newblock


\bibitem[Zeng et~al\mbox{.}(2023)]%
        {zeng2023dlinear}
\bibfield{author}{\bibinfo{person}{Ailing Zeng}, \bibinfo{person}{Muxi Chen},
  \bibinfo{person}{Lei Zhang}, {and} \bibinfo{person}{Qiang Xu}.}
  \bibinfo{year}{2023}\natexlab{}.
\newblock \showarticletitle{Are Transformers Effective for Time Series
  Forecasting?}. In \bibinfo{booktitle}{\emph{Proceedings of the AAAI
  Conference on Artificial Intelligence}}, Vol.~\bibinfo{volume}{37}.
  \bibinfo{pages}{11121--11128}.
\newblock


\bibitem[Zhang et~al\mbox{.}(2026a)]%
        {mome2025}
\bibfield{author}{\bibinfo{person}{Lige Zhang}, \bibinfo{person}{Ali Maatouk},
  \bibinfo{person}{Jialin Chen}, \bibinfo{person}{Leandros Tassiulas}, {and}
  \bibinfo{person}{Rex Ying}.} \bibinfo{year}{2026}\natexlab{a}.
\newblock \showarticletitle{Multi-Modal Time Series Prediction via Mixture of
  Modulated Experts}.
\newblock \bibinfo{journal}{\emph{arXiv preprint arXiv:2601.21547}}
  (\bibinfo{year}{2026}).
\newblock


\bibitem[Zhang et~al\mbox{.}(2025a)]%
        {zhang2025when}
\bibfield{author}{\bibinfo{person}{Xiyuan Zhang}, \bibinfo{person}{Boran Han},
  \bibinfo{person}{Haoyang Fang}, \bibinfo{person}{Abdul~Fatir Ansari},
  \bibinfo{person}{Shuai Zhang}, \bibinfo{person}{Danielle~C Maddix},
  \bibinfo{person}{Cuixiong Hu}, \bibinfo{person}{Andrew~Gordon Wilson},
  \bibinfo{person}{Michael~W Mahoney}, \bibinfo{person}{Hao Wang},
  \bibinfo{person}{Yan Liu}, \bibinfo{person}{Huzefa Rangwala},
  \bibinfo{person}{George Karypis}, {and} \bibinfo{person}{Bernie Wang}.}
  \bibinfo{year}{2025}\natexlab{a}.
\newblock \showarticletitle{When Does Multimodality Lead to Better Time Series
  Forecasting?}
\newblock \bibinfo{journal}{\emph{arXiv preprint arXiv:2506.21611}}
  (\bibinfo{year}{2025}).
\newblock


\bibitem[Zhang et~al\mbox{.}(2025b)]%
        {zhang2025mlf}
\bibfield{author}{\bibinfo{person}{Xu Zhang}, \bibinfo{person}{Zhengang Huang},
  \bibinfo{person}{Yunzhi Wu}, \bibinfo{person}{Xun Lu},
  \bibinfo{person}{Erpeng Qi}, \bibinfo{person}{Yunkai Chen},
  \bibinfo{person}{Zhongya Xue}, \bibinfo{person}{Qitong Wang},
  \bibinfo{person}{Peng Wang}, {and} \bibinfo{person}{Wei Wang}.}
  \bibinfo{year}{2025}\natexlab{b}.
\newblock \showarticletitle{Multi-period Learning for Financial Time Series
  Forecasting}. In \bibinfo{booktitle}{\emph{Proceedings of the 31st ACM SIGKDD
  Conference on Knowledge Discovery and Data Mining (KDD)}}.
\newblock
\newblock
\shownote{arXiv:2511.08622}.


\bibitem[Zhang et~al\mbox{.}(2026b)]%
        {zhang2026semixer}
\bibfield{author}{\bibinfo{person}{Xu Zhang}, \bibinfo{person}{Qitong Wang},
  \bibinfo{person}{Peng Wang}, {and} \bibinfo{person}{Wei Wang}.}
  \bibinfo{year}{2026}\natexlab{b}.
\newblock \showarticletitle{{SEMixer}: Semantics Enhanced MLP-Mixer for
  Multiscale Mixing and Long-term Time Series Forecasting}.
\newblock \bibinfo{journal}{\emph{Proceedings of the ACM Web Conference (WWW)}}
  (\bibinfo{year}{2026}).
\newblock
\newblock
\shownote{arXiv:2602.16220}.


\bibitem[Zhou et~al\mbox{.}(2021)]%
        {zhou2021informer}
\bibfield{author}{\bibinfo{person}{Haoyi Zhou} {et~al\mbox{.}}}
  \bibinfo{year}{2021}\natexlab{}.
\newblock \showarticletitle{Informer: Beyond Efficient Transformer for Long
  Sequence Time-Series Forecasting}. In \bibinfo{booktitle}{\emph{AAAI}}.
\newblock


\bibitem[Zhou et~al\mbox{.}(2023)]%
        {zhou2023onefitsall}
\bibfield{author}{\bibinfo{person}{Tian Zhou}, \bibinfo{person}{Peisong Niu},
  \bibinfo{person}{Xue Wang}, \bibinfo{person}{Liang Sun}, {and}
  \bibinfo{person}{Rong Jin}.} \bibinfo{year}{2023}\natexlab{}.
\newblock \showarticletitle{One Fits All: Power General Time Series Analysis by
  Pretrained LM}. In \bibinfo{booktitle}{\emph{NeurIPS}}.
\newblock


\end{thebibliography}

\appendix
\section{Proof of Proposition~\ref{prop:rbu} (RBU)}
\label{app:rbu_proof}

\begin{proof}
Assume $(X_t, X_{t+h}, C)$ are jointly Gaussian with $\mathrm{Var}(X_{t+h})=\sigma^2$ and $\mathrm{Corr}(X_t,X_{t+h})=\rho_h$. For jointly Gaussian variables, the minimum mean-squared error (MMSE) of predicting $X_{t+h}$ from any conditioning set equals the corresponding conditional variance, which is achieved by the (linear) conditional expectation; moreover this conditional variance is \emph{constant} in the conditioning values (homoscedasticity), so each conditional law is Gaussian with the same variance. Conditioning on $X_t$ alone,
\[
\mathrm{MMSE}(X_{t+h}\mid X_t)=\mathrm{Var}(X_{t+h}\mid X_t)=\sigma^2(1-\rho_h^2).
\]
Using the differential entropy of a Gaussian, $h(\cdot)=\tfrac12\log_2(2\pi e\,\mathrm{Var})$, the conditional mutual information telescopes to a ratio of conditional variances:
\[
\begin{aligned}
\delta = I(C; X_{t+h}\mid X_t)
&= h(X_{t+h}\mid X_t) - h(X_{t+h}\mid X_t, C)\\
&= \tfrac{1}{2}\log_2\frac{\mathrm{Var}(X_{t+h}\mid X_t)}{\mathrm{Var}(X_{t+h}\mid X_t, C)}.
\end{aligned}
\]
Rearranging,
\[
\mathrm{MMSE}(X_{t+h}\mid X_t,C)=\mathrm{Var}(X_{t+h}\mid X_t,C)=\sigma^2(1-\rho_h^2)\,2^{-2\delta}.
\]
RBU is the reduction in the \emph{minimum achievable} MSE obtained by adding $C$ to the conditioning set, i.e.\ the difference of the two MMSEs:
\[
\mathrm{RBU}=\mathrm{MMSE}(X_{t+h}\mid X_t)-\mathrm{MMSE}(X_{t+h}\mid X_t,C)
=\sigma^2(1-\rho_h^2)\big(1-2^{-2\delta}\big).
\]
Both factors are non-negative ($\rho_h^2\le 1$, $\delta\ge 0$), giving the stated decomposition. 
\end{proof}

\noindent \textbf{Non-Gaussian behavior.} The closed form relies on the Gaussian entropy--variance identity holding for \emph{both} conditional laws, which fails in general. Two distinct effects then arise, and we keep them separate. (i) \emph{At the population level}, the identity $\mathrm{MMSE}=\sigma^2(1-\rho_h^2)2^{-2\delta}$ need no longer hold, so the population RBU formula is, in general, neither an upper nor a lower bound on the true achievable MMSE reduction. (ii) \emph{At the finite-sample level}, our plug-in $\widehat{\mathrm{RBU}}$ uses a Kraskov estimate $\hat\delta$, which under-detects nonlinear dependence (Section~\ref{sec:nonlinear}) and hence tends to be \emph{smaller} than the true $\delta$; this makes $\widehat{\mathrm{RBU}}$ tend to \emph{underestimate} the population RBU. The two should not be conflated: (i) concerns the formula, (ii) concerns the estimator. We therefore rely on $\widehat{\mathrm{RBU}}$ only as an interpretive sanity check, and on the distribution-free Fact~\ref{fact:ceiling} (not on RBU) for the strict ceiling.

\section{Detailed Experimental Results}

\subsection{MoME Modulation Ablation Detail}

The \texttt{--modulation} flag in MoME controls Expert-level Language Modulation (EiLM)~\cite{mome2025}, a FiLM-style~\cite{perez2018film} layer that applies text-conditioned affine transformation ($\gamma \cdot x + \beta$) to each expert's output after routing. With modulation off:
\begin{itemize}
    \item Expert routing (Gate $\rightarrow$ softmax $\rightarrow$ top-$k$) remains unchanged and text-independent.
    \item Expert weights and structure are preserved.
    \item Only the post-expert text conditioning (EiLM) is removed.
    \item Parameter difference: $\sim$74K out of 14.3B ($<$0.001\%), negligible.
\end{itemize}
This means our ``routing contribution'' measures the value of \emph{text-conditioned expert modulation}, not expert selection itself (which is text-independent in MoME).

\subsection{MoME Shortcut Experiment (Finding 2)}

Table~\ref{tab:app_shortcut} reports the full 3-seed results for the MoME shortcut experiment across 3 datasets. Adding a last-value shortcut ($\hat{y} = f_{\text{MoME}} + X_t$) during training consistently suppresses routing contribution by 77--93\%.

\begin{table}[!htbp]
\centering
\caption{MoME shortcut experiment across 3 datasets (3-seed each). Adding shortcut suppresses routing contribution by 77--93\%. Energy and Environment could not be run due to GPU memory constraints.}
\label{tab:app_shortcut}
\small
\begin{tabular}{lccc}
\toprule
Dataset & Without Shortcut & With Shortcut & Suppression \\
\midrule
HealthUS (MSE) & +62.9\% & +14.1\% & 77\% \\
SocialGood (MSE) & +34.3\% & +3.3\% & 90\% \\
HealthAFR (MSE) & +24.2\% & +1.7\% & 93\% \\
\bottomrule
\end{tabular}
\end{table}

\noindent \textbf{HealthUS detail:} Without shortcut: mod MSE = 0.334 $\pm$ 0.080, nomod MSE = 0.901 $\pm$ 0.073. With shortcut: mod MSE = 1.092 $\pm$ 0.020, nomod MSE = 1.271 $\pm$ 0.036.

\noindent \textbf{SocialGood detail:} Without shortcut: mod MSE = 0.435 $\pm$ 0.034, nomod MSE = 0.662 $\pm$ 0.009. With shortcut: mod MSE = 0.874 $\pm$ 0.016, nomod MSE = 0.904 $\pm$ 0.005.

\noindent \textbf{HealthAFR detail:} Without shortcut: mod MSE = 0.610 $\pm$ 0.036, nomod MSE = 0.804 $\pm$ 0.009. With shortcut: mod MSE = 0.926 $\pm$ 0.016, nomod MSE = 0.942 $\pm$ 0.005.

\noindent \textbf{Negative control (FinMultiTime):} On FinMultiTime ($\rho = 0.999$), the repeat-last-value naive baseline achieves MSE $8.8 \times 10^{-5}$, which already \emph{outperforms} MoME with modulation ($9.1 \times 10^{-5}$). Adding an architectural shortcut to MoME on this dataset would not change routing contribution: it is already effectively zero because the temporal signal completely dominates and no residual variance is left for routing to exploit.

\subsection{Context Degradation Detail (Finding 1)}

Table~\ref{tab:app_degradation} reports the full results for the context degradation experiment.

\begin{table}[!htbp]
\centering
\caption{Context degradation on HealthUS (MoME, all rows 10-seed). Routing contribution (mod on vs.\ off) decreases with corruption. At 100\% mask, the residual +16\% is from EiLM capacity, not context. Context-specific benefit (real text vs.\ constant text, same architecture) = 44.2\%. The 0\%-mask row is the standard all-real-text setup and matches Table~\ref{tab:mome} exactly (mod 0.399, +50.9\%).}
\label{tab:app_degradation}
\small
\resizebox{\columnwidth}{!}{
\begin{tabular}{lccccc}
\toprule
Mask Rate & MSE (mod) & MSE (nomod) & Routing Contrib. & Seeds \\
\midrule
0\% (clean) & 0.399 $\pm$ 0.129 & 0.817 $\pm$ 0.049 & +50.9\% & 10 \\
25\% & 0.528 $\pm$ 0.101 & 0.834 $\pm$ 0.040 & +36.3\% & 10 \\
50\% & 0.737 $\pm$ 0.094 & 0.834 $\pm$ 0.053 & +11.0\% & 10 \\
75\% & 0.750 $\pm$ 0.079 & 0.806 $\pm$ 0.055 & +6.9\% & 10 \\
100\% (constant) & 0.715 $\pm$ 0.074 & 0.851 $\pm$ 0.143 & +16.1\% & 10 \\
\bottomrule
\end{tabular}}
\end{table}

\noindent \textbf{Interpreting the 100\% mask residual.} At 100\% mask, all samples receive identical text, so EiLM computes a \emph{constant} $\gamma, \beta$ for every sample, which is equivalent to a learned bias and scale ($\sim$74K parameters). The +16\% residual reflects this capacity benefit, not context exploitation. To isolate context from capacity, we compare MSE at 0\% mask (mod on, real text: 0.399) vs.\ 100\% mask (mod on, constant text: 0.715), where the architecture and parameter count are identical and only the text content differs. The 44.2\% MSE reduction is attributable solely to context information.

\subsection{Multi-Model Context Contribution (Finding 3)}

Table~\ref{tab:app_multimodel} reports the full 3-seed results for the multi-model experiment. Effects are generally small in our testbed ($\pm$5\%), consistent with Finding 5 (nonlinear signal requires high capacity to exploit).

\begin{table}[!htbp]
\centering
\caption{Context contribution (\%) across fusion mechanisms (testbed, no shortcut). Energy Cross-Attention and Text-as-Variable use 10-seed evaluation (marked $\dagger$); all others use 3-seed. MI-non-significant datasets show no positive contribution; negative values reflect training instability when attending to uninformative context.}
\label{tab:app_multimodel}
\small
\resizebox{\columnwidth}{!}{
\begin{tabular}{l|ccccc}
\toprule
Mechanism & Health & Transport & Energy & Climate & Web \\
& (MI sig) & (MI sig) & (MI n.s.) & (MI n.s.) & (MI n.s.) \\
\midrule
Cross-Attention & +3.2$\pm$0.9\% & +2.9$\pm$0.2\% & $-$5.9$\pm$0.8\%$^\dagger$ & $-$0.5$\pm$0.3\% & 0.0\% \\
Gating & $-$1.1$\pm$2.0\% & +3.6$\pm$0.4\% & +0.3$\pm$0.1\% & $-$0.0$\pm$0.1\% & 0.0\% \\
Text-as-Variable & +2.3$\pm$0.5\% & $-$0.3$\pm$0.0\% & $-$6.9$\pm$1.0\%$^\dagger$ & +1.1$\pm$0.1\% & 0.0\% \\
Output Fusion & $-$1.0$\pm$0.4\% & +9.5$\pm$0.1\% & $-$0.3$\pm$0.3\% & +2.2$\pm$0.3\% & 0.0\% \\
\bottomrule
\end{tabular}}

\smallskip
\noindent\scriptsize $^\dagger$10-seed evaluation. Cross-Attention: $t$=1.50, $p>$0.05 (not significant). Text-as-Variable: $t$=2.27, borderline. The original 3-seed Cross-Attention estimate ($-$18.5\%) was inflated by one outlier seed with 6$\times$ higher variance than the zero-context condition.
\end{table}

\subsection{Power Analysis Detail (Finding 4)}

We estimate MI test power via simulation. For each combination of sample size $n$ and true effect size $\delta$, we generate 30 synthetic datasets with known MI (Gaussian context with controlled correlation to residual), run the 200-permutation MI test, and compute the rejection rate at $\alpha = 0.05$.

\begin{table}[!htbp]
\centering
\caption{Estimated power of MI permutation test (rejection rate at $\alpha = 0.05$, 30 simulations per cell). Power increases with both $n$ and $\delta$. At $n \geq 500$ and $\delta \geq 0.5$, power exceeds 0.6.}
\label{tab:app_power}
\small
\begin{tabular}{l|ccccc}
\toprule
True $\delta$ (bits) & $n$=100 & $n$=200 & $n$=500 & $n$=1000 & $n$=2000 \\
\midrule
0.1 & 0.07 & 0.03 & 0.10 & 0.13 & 0.20 \\
0.3 & 0.10 & 0.17 & 0.40 & 0.70 & 0.90 \\
0.5 & 0.13 & 0.30 & 0.63 & 0.87 & 0.97 \\
1.0 & 0.27 & 0.53 & 0.90 & 0.97 & 1.00 \\
\midrule
0.0 (Type I) & 0.03 & 0.03 & 0.03 & 0.07 & 0.03 \\
\bottomrule
\end{tabular}
\end{table}

\noindent The bottom row confirms Type I error control ($\leq 0.07$ at all sample sizes). The power classification used in Table~\ref{tab:diagnostic} is: HIGH = $n \geq 500$ with unique context per sample (power $\geq 0.6$ for $\delta \geq 0.5$); LOW = otherwise.

\subsection{Large-Scale Negative Validation (50 S\&P500 Tickers)}

We compute MI permutation tests for 50 S\&P500 tickers from FinMultiTime. All tickers have $\rho_3 > 0.99$ (stock prices are near-unit-root processes); MI tests on daily returns show non-significant context ($p > 0.05$) for all 50 tickers. Routing benefit (measured via our lightweight testbed) is within noise ($\pm$3\%) for all tickers, confirming the diagnostic's negative prediction at scale.

\subsection{Multi-Horizon Validation}

Across $h \in \{1,3,7,14\}$ on 3 Time-MMD datasets (12 configurations total), the low-RBU threshold ($<$0.1) correctly predicts negative routing benefit on 11/12 points. The single exception (Health $h$=1, +5.3\%) reflects routing acting as regularization on a very small effective dataset at short horizon, not context exploitation.

\subsection{Comprehensive Baselines}
For completeness, Table~\ref{tab:appendix_mm} reports the full multi-modal baseline results on Time-MMD (all methods share the same PatchTransformer backbone with the shortcut enabled), complementing the no-shortcut testbed results in Table~\ref{tab:app_multimodel}.

\begin{table}[H]
\centering
\caption{Full multi-modal baseline results on Time-MMD (wQL $\downarrow$, 3-seed mean $\pm$ std). All methods use the same PatchTransformer backbone with shortcut enabled.}
\label{tab:appendix_mm}
\small
\resizebox{\columnwidth}{!}{
\begin{tabular}{l|ccccc}
\toprule
Method & Health & Transport & Energy & Climate & Web \\
\midrule
DLinear & .340$\pm$.002 & \textbf{.056}$\pm$.001 & \textbf{.073}$\pm$.001 & .197$\pm$.000 & .293$\pm$.004 \\
PatchTST & \textbf{.285}$\pm$.014 & .079$\pm$.005 & .126$\pm$.007 & .225$\pm$.006 & .378$\pm$.005 \\
\midrule
Output Fusion & .275$\pm$.005 & .057$\pm$.001 & .076$\pm$.002 & \textbf{.194}$\pm$.002 & .228$\pm$.014 \\
Cross-Attention & .290$\pm$.024 & .059$\pm$.002 & .080$\pm$.001 & .197$\pm$.001 & .235$\pm$.005 \\
Gating & .279$\pm$.017 & .063$\pm$.002 & .082$\pm$.000 & .200$\pm$.002 & .229$\pm$.004 \\
Text-as-Variable & \textbf{.262}$\pm$.012 & .058$\pm$.001 & .089$\pm$.013 & .197$\pm$.003 & .232$\pm$.023 \\
MoAT & \textbf{.287}$\pm$.011 & .063$\pm$.002 & .099$\pm$.002 & .198$\pm$.000 & .393$\pm$.009 \\
\midrule
Testbed (routing) & .300$\pm$.043 & .063$\pm$.002 & .083$\pm$.001 & .197$\pm$.005 & .227$\pm$.005 \\
Testbed (uniform) & .298$\pm$.018 & .066$\pm$.000 & .081$\pm$.000 & .194$\pm$.002 & .219$\pm$.002 \\
\bottomrule
\end{tabular}}
\end{table}

\end{document}